%% file: turning_machines.tex
\documentclass[a4paper]{article}

\usepackage[bottom=3cm,top=3cm]{geometry}

\usepackage{graphicx,floatrow}
\usepackage{amsfonts,amssymb,amsmath,amsthm}
\usepackage{color,comment,wrapfig}
\usepackage[normalem]{ulem}
\usepackage[shortlabels]{enumitem}
\usepackage{url}
\usepackage{hyperref}
\hypersetup{colorlinks=false,pdfborder={0 0 0}}
\usepackage{thm-restate}
\usepackage{microtype,xspace}

\usepackage[a]{esvect}  

\usepackage[color]{changebar}

\usepackage[linecolor=black,textsize=tiny,backgroundcolor=white]{todonotes}

\renewcommand{\vec}{\vv}
\newcommand{\vx}{\vec{x}}
\newcommand{\vy}{\vec{y}}
\newcommand{\vw}{\vec{w}}
\newcommand{\tzz}{\ensuremath{T^{\textrm{zz}}}\xspace}

\newcommand{\Nset}{\mathbb{N}}
\newcommand{\Zset}{\mathbb{Z}}

\newcommand{\pos}{\mathrm{pos}}
\newcommand{\head}{\mathrm{head}}  
\newcommand{\headp}{\mathrm{head}^{\rightarrow}\!}
\newcommand{\tail}{\mathrm{tail}}  

\newcommand{\tml}[2]{\ensuremath{L_{#1}^{#2}} }

\newcommand{\ra}{\rightarrow}

\newtheorem{theorem}{Theorem}

\newtheorem{corollary}[theorem]{Corollary}{\bfseries}{\itshape}
\newtheorem{lemma}[theorem]{Lemma}{\bfseries}{\itshape}
\newtheorem{proposition}[theorem]{Proposition}{\bfseries}{\itshape}
\newtheorem{definition}[theorem]{Definition}{\bfseries}{\itshape}
\newtheorem{remark}[theorem]{Remark}{\bfseries}{}

\newtheorem{openproblem}[theorem]{Open problem}

\title{Turning Machines:\\ a simple algorithmic model for molecular robotics\footnote{This paper expands upon the conference version published in {\em Proceedings of the 26th International Conference on DNA Computing and Molecular Programming}. Oxford 2020, Leibniz International Proceedings in Informatics (LIPIcs). 
Authors C.\ Wood and D.\ Woods are supported by European Research Council (ERC) award number 772766 and Science foundation Ireland (SFI) grant 18/ERCS/5746 (this manuscript reflects only the authors' view and the ERC is not responsible for any use that may be made of the information it contains).}  \vspace{1ex}} 

\author{Irina Kostitsyna\thanks{Department of Mathematics and Computer Science, TU Eindhoven, the Netherlands.  i.kostitsyna@tue.nl}
	 \and
        Cai Wood\thanks{Hamilton Institute and Department of Computer Science, Maynooth University, Ireland. \{cai.wood.2017@mumail.ie, damien.woods@mu.ie\}   } 
        \and
        Damien Woods\footnotemark[3]
        \vspace{1ex}
}

\date{}

\begin{document}

\maketitle
\begin{abstract}
Molecular robotics is challenging, so it seems best to keep it simple.  We consider an abstract molecular robotics model based on simple folding instructions that execute asynchronously.  Turning Machines are a simple 1D to 2D folding model, also easily generalisable to 2D to 3D folding.  A Turning Machine starts out as a line of connected monomers in the discrete plane, each with an associated turning number.  A monomer turns relative to its neighbours, executing a unit-distance translation that drags other monomers along with it, and through collective motion the initial set of monomers eventually folds into a programmed shape.  We provide a suite of tools for reasoning about Turning Machines by fully characterising their ability to execute line rotations:  executing an almost-full line rotation of $5\pi/3$ radians is possible, yet a full $2\pi$ rotation is impossible.  Furthermore, line rotations up to $5\pi/3$ are executed efficiently, in $O(\log n)$ expected time in our continuous time Markov chain time model.  We then show that such line-rotations represent a fundamental primitive in the model, by using them to efficiently and asynchronously fold shapes.  In particular, arbitrarily large zig-zag-rastered squares and zig-zag paths are foldable, as are $y$-monotone shapes albeit with error (bounded by perimeter length).  Finally, we give shapes that despite having paths that traverse all their points, are in fact impossible to fold, as well as techniques for folding certain classes of (scaled) shapes without error.  Our approach relies on careful geometric-based analyses of the feats possible and impossible by a very simple robotic system, and  pushes conceptional hardness towards mathematical analysis and away from molecular implementation.
\end{abstract}

\section{Introduction}
The challenge of building molecular robots has many moving parts,  
as the saying goes. 
These include   
 molecular parts that move relative to each other;  
 units needing some sort of memory state; 
 the ability to transition between states; 
and perhaps even the ability to use computation to drive robotic movements. 
Here we consider a simple algorithmic model of robotic reconfiguration called Turning Machines.

The main ethos behind our work is the notion of having a reconfigurable structure where component monomers actuate their position relative to their neighbours and  
 governed by {\em simple} actuation rules.
 Volume exclusion applies (two monomers can not occupy the same position in space), almost for free we get massive parallelism and asynchronicity,
 and the complexity of allowable state changes is small: unit monomers start with a natural number and decrement step-by-step to zero.
 The Turning Machine model  embodies these concepts. 

On the one hand, there are a number of senses in which molecular systems are {\em better} suited to robotic-style reconfiguration than macro-scale robotic systems:  there is no gravity nor friction fighting against components' actuation, and should we know how to exploit them, randomness, freely diffusing fuel (robots need not carry all their fuel) and large numbers of components are all readily available as resources. 
On the other hand, building nanoscale components presents a number of challenges including implementing computational controllers at the nanoscale, as well as designing systems that self-assemble and interact in a regime where we can not easily send in human mechanics to diagnose and fix problems.

\begin{figure}[t]
  \centering
\includegraphics{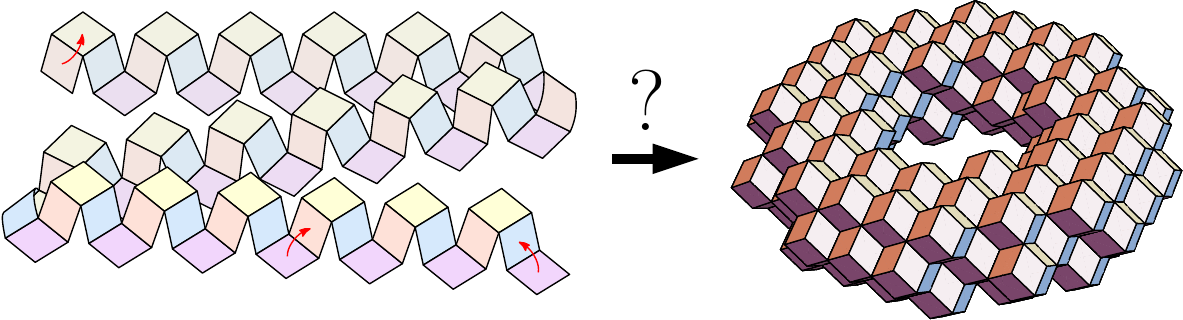}
\caption{Turning Machine motivation: what shapes can be made by autonomously folding structures using simple local turning rules that effect non-local movement?
Finding suitable abstract models and characterising their ability helps us to step back and create a vision of where we can go.}
    \label{fig:motivation}
\end{figure}

\subsection{Turning Machines}
Monomers are the atomic components of a Turning Machine and are arranged in a connected chain on the triangular grid $G_\triangle$, with each monomer along the chain pointing at the next (Figure~\ref{fig:model}).  In an initial instance, the chain of monomers are sitting on the $x$-axis all pointing to the east.  Each monomer has an initial (input) integer turning number $s\in\Zset$, the monomer's ultimate goal is to set that number to 0:  if $s$ is positive, the monomer tries to simultaneously decrement~$s$ and turn anti-clockwise 
by an angle of $\pi/3$ (Figure~\ref{fig:model}(b)), if $s$ is negative, it tries to increment and turn clockwise by $\pi/3$.\footnote{Having the monomer turning angle 
be confined to the range $(0, \pi/2]$ seems to capture a range of interesting and important blocking behaviours that would otherwise be missed by the model. Having the angle be 
$\pi/3$, and in particular choosing the triangular grid over the square grid, is a somewhat arbitrary choice in the model definition.} If $s=0$ the monomer has reached its target orientation and does not turn again (Figure~\ref{fig:model}(b)).

A key point is that although a monomer actuates by rotating the direction in which it points, when it does so it ``drags'' (translates) all monomers that come after it in the chain in the same way 
the rotation motion of a human arm (around a shoulder) appears to translate a flag through the air, or the way a cam in a combustion engine converts rotational shaft motion to translational piston motion (Figures~\ref{fig:model}(c) and~\ref{fig:60}).  Section~\ref{sec:def} gives the precise model definition. 

\begin{figure}[t]
  \centering
    \includegraphics[page=1, height=3.3cm]{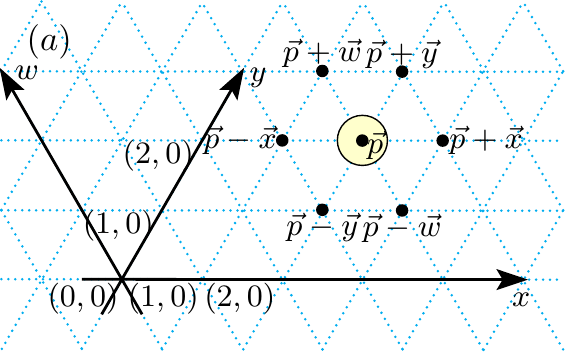}
    \hfill
 \includegraphics[page=2,height=3.3cm]{images/figures-new.pdf}
 
 \vspace{0.2cm}
 
 \includegraphics[page=15,width = \textwidth]{images/figures-new.pdf}
 
 \vspace{0.2cm}
 
 \includegraphics[page=16,width=\textwidth]{images/figures-new.pdf}
  \caption{Turning Machine model. 
  (a) Triangular grid conventions.   A configuration showing a single monomer on the triangular grid $G_\triangle$, along with axes $x$, $y$ and $w$.
  (b) Local movement (rotation): A monomer in state 3 pointing to the east undergoes three turning rule applications finishing in state 0 with no more applicable rules. Locally, the monomer effects a rotation motion.
(c) Global movement (translation):  
Left: Example initial configuration. 
Middle: A movement rule applied to a monomer in state 3 decrements the state to 2 and translates all subsequent monomers by $+\protect\vw$. 
Right: A second state-3 monomer moves.
(d) Blocking: Monomers are not permitted to make movements that would result in a self-intersecting configuration. 
The monomer $m_j$ is said to be {\em blocked} because if it were to move, then $m_k$ would move to overlap $m_i$, but such self-intersections are not permitted.}
    \label{fig:model}
\end{figure}

\subsubsection{The main challenge: blocking}
Programming the model simply requires annotating an east-pointing line of monomers with turning numbers; an incredibly simple programming syntax. 

Locally, individual monomers exhibit a small rotation, but globally this effects a large translation, or dragging, of many monomers (Figure~\ref{fig:model}(c)). 
Thus globally, the main challenge is how to effect global rotations -- in other words how to use translation to simulate rotation. In particular, how to do this when lots of monomers are asynchronously moving and bumping into each other, potentially blocking each other from moving. 

Blocking is illustrated in Figure~\ref{fig:model}(d) and comes in two forms. {\em Temporary blocking} where one monomer is in the way of another, but eventually will get out of the way, and {\em permanent blocking} where all monomers block each other in a locked configuration that will never free itself. 
We say that a target structure is foldable if all possible system trajectories  lead to that structure, i.e. permanent  blocking does not occur on any trajectory.
A foldable structure may exhibit temporary blocking on some trajectories, indeed most of the work for our positive results in this paper comes down to showing that for certain folding tasks, any blockings that happen are merely temporary kinks in the chain that are eventually worked out. 
We measure the amount of blocking by considering the completion time: a foldable structure where temporarily blocked monomers can quickly become unblocked finishes faster than one where blocking takes a while to sort out.  
Our model of time assumes that the time to apply a turning rule to a given unblocked monomer is an exponential random variable with rate~1, and the system evolves as a continuous time Markov chain with the discrete events being rules applied asynchronously and in parallel. 

Given a shape and a best-guess at an initial (input) state sequence to fold that shape, the main challenge often lies in showing that all trajectories lead to the desired shape, i.e. that permanent blocking never occurs. If we succeed in this, a second challenge is to analyse the time to completion---with speed of completion being a measure of efficiency and parallelism.

\subsection{Results: line rotations and shape building}
After defining the Turning Machines model (Section~\ref{sec:def}), and giving technical lemmas for reasoning about it   (Section~\ref{sec:tools}),   
in Section~\ref{sec:lines} we fully characterise the line rotation capabilities of the model, in two senses.
First, we show that for each of the angles $\theta \in \{ \pi/3, 2\pi/3, \pi, 4\pi/3, 5\pi/3 \}$, and any number of monomers $n\in\Nset$ there is a Turning Machine with $n$ monomers that starts on the $x$-axis and ends rotated by $\theta$ radians. We show this is the best one can do, that is, that rotation of $\theta \geq 2\pi$ is impossible (for any $ n>7$, there are always some trajectories that are permanently blocked). 
Second, line rotation is fast:  Up to constant factors the speed is optimal, completing in expected time $O(\log n)$. 
This shows that despite the fact that line rotations in the range $\pi \leq \theta \leq 5\pi/3$ experience large number of blockings along their trajectories, these blockings are all temporary, and do not conspire to slow  the system down by more than a constant factor on average.

In Section~\ref{sec:shapes}, we go on to give results on the shape-building abilities of Turning Machines. 
To do this, we build on previous  technical tools to show that so-called zig-zag paths (paths that wind over and back parallel to the $x$ axis) are foldable, and in merely logarithmic expected time (Theorem~\ref{thm:zz}). 
We use that result to prove that $n \times n$ squares, rastered in a zig-zag fashion, are foldable, and time-efficiently so  (Corollary~\ref{cor:square}). 
By allowing error in the folding of a shape (the error is the symmetric difference between the set of points of the shape and points in the folding), 
any shape is trivially foldable, albeit with error up to the area of the shape, by simply folding the entire bounding box of the shape.\footnote{Or even just drawing a line through a shape, or even just a single point in the shape, or even zero points in the shape!} 
This means that there are shapes, that are not Hamiltonian but are foldable with large error  (e.g. a cross with thin arms, Figure~\ref{fig:cross} and Theorem~\ref{thm:cross}). 
What shapes can be folded with small error?  We show that any shape from a wide class called $y$-monotone shapes is foldable in optimal\footnote{We don't state it as a result, but the folding tasks in this paper have an expected time lower bound of $\Omega(\log n)$, for a length $n$ Turning Machine instance, since they require they require $\geq n$ events to occur.}
expected time $O(\log n)$, and with error no more than the perimeter of the shape and no more than the perimeter of the folding (Theorem~\ref{thm:y-monotone-scale1}). 
Certain classes of shapes can be folded with zero error:  
Theorem~\ref{thm:yw-separator} shows that there are shapes, that we call,  $y$-monotone with a $yw$-separator, that are are foldable at scale factor~2. 
Finally, one can ask if every shape that has a Hamiltonian path is foldable. The answer is no:  our main negative result (Theorem~\ref{thm:spiral})  gives a classes of shapes (thin spirals with a gap between their arms) and proves that they are not foldable. 

We would argue that programming in this model is deceptively simple: if a shape has a  Hamiltonian path, it is typically straightforward to trace that  path  while assigning turning numbers to an initial configuration such that there is {\em some} trajectory that folds the shape. 
However, making the step forward to then show that {\em all} trajectories fold that shape may involve significant and subtle argumentation, or may be impossible. 
Programming in the model seems simple, the analysis may not be. This has the benefit of pulling hardness away from the molecular robot design and implementation problem (which is a rather challenging experimental problem) and instead pushing it towards a theoretical geometric analysis problem, exactly where we'd prefer it to be.

\subsection{Related and future work}
Besides finding insights at the interface of computation and geometry, another ultimate aim of this kind of work is to bridge the gap between what we can imagine in theory and what we can engineer in the lab~\cite{ramezani2019building}. 
Biological systems actuated at the molecular scale  provide inspiration: 
in the gastrulation phase of embryonic development of the model organism \emph{Drosophila melanogaster}, large-scale rearrangements of the embryo are effected by thousands of (nanoscale) molecular motors working together to rapidly push and pull the embryo into a target shape~\cite{dawes2005folded,Wieschaus08pulsed}.  

 Our Turning Machine model is a restriction of the nubot model~\cite{nubots},  a molecular robotic model with many  features including self-assembly capabilities, 
random agitation (jiggling) of monomers, 
the ability to execute cellular automata style rules,
and floppy/rigid molecular bonds. 
The parallel computing capabilities~\cite{nubotParallel}, 
and construction using random agitation and self-assembly~\cite{nubotAgitation} have been studied. 
Dabby and Chen consider related (experimental and theoretical) systems that use an insertion primitive to quickly grow long (possibly floppy) linear structures~\cite{dabbyChenSODA2012}, later tightly characterised  by  Hescott, Malchik and Winslow~\cite{WinslowInsertion2018,Winslow2017insertionTightBounds} in terms of number of monomer types and time.
Hou and Chen~\cite{hou2019exponentially} show that the nubot model can display exponential growth without needing to exploit  state changes. Chin, Tsai and Chen~\cite{chin2018minimal} look at both minimising numbers of state changes and number of `2D layers' to assembly 1D structures.
There are a number related autonomous self-folding models, both 1D to 2D~\cite{cheung2011programmable} and 2D to 3D~\cite{connelly2010locked},
and reconfigurable robotic/programmable matter systems, e.g.~\cite{aloupis2009linear,aloupis2008reconfiguration,demaine2018know,geary2016programming,gmyr2019forming,michail2019transformation}.

There are several avenues for future work.
\begin{itemize} 
\item Concretely, it remains to  fully characterise the classes of shapes foldable with zero error (Open Problem~\ref{op:foldable shapes}). 
For example, can our main negative result showing the impossibility of folding spirals be generalised to give impossibility result for wider classes of shapes? Can we find new techniques to increase the class of shapes that are foldable (with zero error), beyond the zig-zag rastering techniques we heavily use in Sections~\ref{sec:zz}--\ref{sec:exact-folding} (to fold $n \times n$ squares, and certain other $y$-monotone shapes including those with a $yw$-separator at scale factor 2).
\item Establishing  error bounds for various classes of shapes, including trade-offs between notions of shape complexity and error could be another avenue to explore. It seems interesting to consider folding where we permit a small number of trajectories to be permanently blocked.  
\item 
In this paper, our positive results are mainly for model instances with either positive integer states, which define an anti-clockwise rotation-and-translation motion about the origin, or with negative integer states which define a clockwise motion; our negative results hold for both directions. Does the combination of {\em both} anti-clockwise (positive) and clockwise (negative) turning rules strictly increase the expressivity of the model?
 \item 
Using a variant~\cite{nubots,nubotAgitation} of the model with random agitation of monomers would side-step our negative result about the impossibility of {\em reaching} a full $2\pi$ line rotation, essentially by allowing reversible movement out of blocked configurations (although upon reaching the target configuration, the Turing machine would also immediately reverse out it). We don't know if  adding agitation to the model can help increase the set of foldable shapes, however since it is a somewhat natural physical notion it is worth investigating --- either with or without the notion of locking into position when a monomer has  completed all of its movements by reaching state 0.
Indeed, the analysis of such systems would provide intellectual fruit by mixing probability, geometry and computation.
\item As indicated in Figure~\ref{fig:motivation}, it is straightforward to generalise the model to (say) 2D trees folding into 3D shapes. 
\item We know that certain kinds of zig-zag paths (e.g. over-and-back parallel to x-axis, while moving monotonically either positive or negative along the $y$ axis, see for example Figure~\ref{fig:zz}) are efficiently foldable, and under certain conditions mixing of two kinds of of zig-zag paths are foldable (e.g.~proof of  Theorem~\ref{thm:yw-separator} and Figure~\ref{fig:yw-monotone}).
\end{itemize}
In all of these cases fully characterising the classes of shapes that can be folded, and characterising the time to fold such classes of structures, provides questions whose answers would expand our understanding of the capabilities of simple reconfigurable robotic systems.

\section{Turning Machine model definition}\label{sec:def}
   
In this section we define the Turning Machine model. 
Formally speaking, the model is a restriction of the Nubot model~\cite{nubots}, for simplicity we instead use a custom formalism. 

\paragraph{Grid.}
Positions are pairs in $\mathbb{Z}^2$ defined on a two-dimensional (2D) triangular grid $G_\triangle$ using $x$ and $y$ axes as shown in Figure~\ref{fig:model}. 
We use the notation $\mathbb{N}^2$ to denote points in the positive orthant, or positive sextant, of the 2D triangular grid.
For convenience, we define a third axis, $w$, centred on the origin and running through the point $(x,y) = (-1,1)$. 
We let $\pm\vx, \pm\vy,  \pm\vw$ denote the unit vectors along the $x$, $y$ and $w$ axes. 

\paragraph{Monomer, configuration, trajectory.}
A monomer is a pair $m = (s(m), \pos(m)$) where $s(m)  \in \mathbb{Z}$ is a state and $\pos(m_i)  \in \mathbb{Z}^2 $ is a position.
A {\em configuration}, of length $n\in\mathbb{N}$, is a tuple of monomers $c = ( m_0, m_1, \ldots, m_{n-1})$ 
whose positions $\sigma(c) = \pos(m_0), \pos(m_1), \ldots,$ $\pos(m_{n-1})$  define a length $n-1$ simple directed path (or non-self-intersecting chain)  
in $\mathbb{Z}^2$ (on the grid $G_\triangle$) and where $\pos(m_0) = (0,0)$.\footnote{In the language of~\cite{nubots}, one can imagine that for all $i \in \{0,1,\ldots, n-2\}$,  there is a rigid bond between monomer $m_i$ and monomer $m_{i+1}$, and otherwise there are no bonds.} 

A configuration is a tuple of $n \in \mathbb{N}$ monomers $(m_0,m_1,\ldots, m_{n-1})$.
A {\em final configuration} has all monomers in state 0. 
A pair of configurations $(c_i, c_{i+1})$ is said to be a {\em step} if $c_i$ yields~$c_{i+1}$ via a single {\em rule application} (defined below) which we write as $c_i\ra c_{i+1}$.  
A trajectory, of length $k$, is a sequence of  configurations $c_0, c_1,\ldots,c_{k-1}$ where, for each $i\in \{0,1,\ldots,k-2\}$ the pair $(c_i, c_{i+1})$ is a step $c_i \ra c_{i+1}$. 
A Turning Machine {\em initial configuration} $c_0$ is said to {\em compute the target configuration}~$c_t$ if all trajectories that start at $c_0$ lead to $c_t$, and is said to compute its target configuration if it reaches the configuration with all monomers in state 0.
A {\em Turning Machine} instance is an initial configuration.
For a monomer $m_i$, we let $s_0(m_i)$ denote its state in the initial configuration, and we let $\ell_i$ denote the horizontal line through $\pos(m_i)$. 
Given a monomer $m_i$ in configuration $c$ we say that for $0\leq i \leq n-2$ monomer $m_i$  {\em points in direction} $\vec d \in \{\pm\vx, \pm\vy, \pm\vw \}$ if $\pos (m_{i+1}) - \pos(m_i) = \vec d$. By convention, monomer $m_{n-1}$ does not point in any direction, and whenever we say {\em all monomers point in some direction $\vec d$}, we mean all except $m_{n-1}$. 
A configuration $c'$ is {\em reachable} from configuration $c$ if there is at least one sequence of rule applications from $c$ to $c'$ (if $c$ is not specified we mean the initial configuration $c_0$).

\paragraph{Turning rule: state decrement/increment.}
Let $S_{\mathrm{init}} \subsetneq \mathrm{Z} $ be the set of states that appear in the initial configuration. 
Let $s_{\min} =  \min({ S_{\mathrm{init}} \cup \{ 0 \} }) $  and 
$s_{\max} =  \max({ S_{\mathrm{init}} \cup \{ 0 \} }) $, and let 
$S = \{s_{\min}, s_{\min}+1,\ldots , s_{\max} \} $ be the called the Turning Machine {\em state set}.  
The {\em turning rules} of a Turning Machine are  
defined by a function $r$ such that for all states $s \in ( S \setminus \{ 0 \} )$: 
\begin{equation}\label{eq:r} 
r(s) = 
 \begin{cases}
 s -1  \quad\textrm{ if } s >0\,,\\
 s +1  \quad\textrm{ if } s < 0\,.
 \end{cases}
\end{equation}
Let $\mathcal{C}$ be the set of all configurations. 
The turning rule $R : \mathcal{C} \times \mathbb{Z} \rightarrow \mathcal{C} $ is a function and  $R(c,i)$   is said to be {\em applicable} to monomer $m_i$ in configuration $c$ if $s(m_i) \neq 0$ and the rule is not blocked 
(defined below). 
If the rule is applicable, we write $R(c,i) = c'$ and say that $R(c,i)$ yields the new configuration $c'$, and we say that $(c,c')$ is a step.

\paragraph{Turning rule: blocking.}
For $i \in \{0,1,\ldots,n-1 \}$, we define the head and tail of monomer $m_i$ as 
$\head(m_i) = m_{i+1},m_{i+2},\ldots,m_{n-1}$
and 
$\tail(m_i) = m_{0},m_{1},\ldots,m_{i}$.
Consider the following tuple of unit vectors: $\vec d = ( \vx, \vy, \vw, -\vx, -\vy, -\vw) $, and let $\vec d_k$ denote the $k$th element of that tuple. 
Let $\vec d_i = \pos(m_{i+1})-\pos(m_i)$ be the direction of monomer $m_i$, and then if $s(m_i)>0$ let $i' = (i+2)~\mathrm{mod}~6$, or if $s(m_i)<0$ let $i' = (i - 2) ~\mathrm{mod}~6$.
For a vector $\vec d\in \Zset^2$ we write $m_i + \vec d$ to mean the monomer $m_i$ translated by $\vec d$. 
Define\footnote{Another way to state this is that when a monomer $m_i$ moves, $\head(m_i)$ translates in the direction corresponding to the current direction of $m_i$ rotated by the angle $2\pi/3$.}  
$\headp(m_i) = 
 m_{i+1} + \vec d_{\!i'},\, m_{i+2}+ \vec d_{\!i'},\, \ldots, m_{n-1} + \vec d_{\!i'}$.
 If the set of positions of $\tail(m_i)$
 has a non-empty intersection with the set of positions of  $\headp(m_i)$
 we say that the rule is blocked, and the rule is not applicable.
 If the rule is not blocked, it is applicable and the resulting next configuration is  
 $c'  = \tail(m_i)$, $\headp(m_{i}) = m_{0},m_{1},\ldots,m_{i}, 
  m_{i+1} + \vec d_{i'}, m_{i+2}+ \vec d_{i'}, \ldots, m_{n-1} + \vec d_{i'}$.
  
A configuration $c$ is said to be {\em permanently blocked} if (a) not all states are 0, and (b) none of the monomers in $c$ has an  applicable rule.
A monomer $m$ within a configuration $c$ is said to be {\em temporarily blocked} if 
(a) $m$ is not in state 0, and 
(b) there is no rule applicable to $m$, and 
(c) there is a trajectory starting at $c$ that reaches a configuration $c'$ where there is a rule applicable to $m$.

\paragraph{Time.}
A Turning Machine evolves as a continuous time Markov process. The rate for each rule application is 1. If there are $k$ applicable transitions for a configuration~$c_i$ (i.e.\ $k$ is the sum of the number of rule applications that can be applied to all monomers in $c_i$), then the probability of any given transition being applied is $1/k$, and the time until the next transition is applied is  an exponential random variable with rate $k$ (i.e.\ the expected time is $1/k$). 
 The probability of a trajectory is then the product of the probabilities of each of the transitions along the trajectory, and the expected time of a trajectory is the sum of the expected times of each transition in the trajectory. Thus, $\sum_{t \in \mathcal{T}} \mathrm{Pr}[t] \cdot  \mathrm{time}(t)$ is the expected time for the system to evolve from configuration~$c_i$ to configuration~$c_j$, where~$\mathcal{T}$ is the set of all trajectories from~$c_i$ to $c_j$, and $ \mathrm{time}(t)$ is the expected time for trajectory $t$.

\paragraph{Example.}
These concepts are illustrated in the proof of Lemma~\ref{lem:pi/3 rotation} in Appendix~\ref{app:sec:rot}, 
and in the example in Figure~\ref{fig:60}. 

\section{Introduction to line rotation Turning Machines}
\begin{figure}[t]
  \centering
    \includegraphics[page=5]{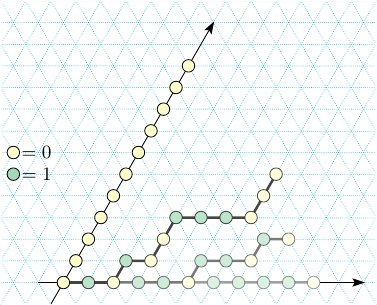}\hfill
    \includegraphics[page=6]{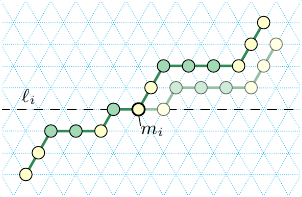}
  \caption{Left: 
  The Turning Machine $\tml{n}{1}$ that rotates a line of $n=11$ monomers by $\pi/3$; illustration for Lemma~\ref{lem:pi/3 rotation}.
  Four configurations are shown. The initial configuration has all monomers in state 1 sitting on the $x$-axis, in the final  configuration all are in state 0 and sitting on the $\pi/3$ line. 
  Two intermediate configurations are shown, respectively after 2, and then after 5,  turning rules applications. 
Right: A configuration of the Line rotation Turning Machine $L^1_{11}$ with the chain running from bottom left to top right. Lemmas~\ref{lem:pi/3 rotation} and~\ref{lem:180:li} uses the fact that $\tail(m_i)$ sits on or below $\ell_i$, $\head(m_i)$ sits on or above $\ell_i$, and $\headp(m_i)$ sits strictly above $\ell_i$.}
    \label{fig:60}
\end{figure}
Every Turning Machine analysed in this paper starts with $n\in\Nset$ monomers, 
sitting on the $x$-axis, pointing to the east. 
We define a class of Turning Machines which we call \emph{line rotation Turning Machines}. 
\begin{definition}[line rotation Turning Machine]\label{def:line}
Let $n\in\Nset$ and let $\tml{n}{\sigma}$ be the Turning Machine with initial configuration of $n$ monomers $c_0 = m_0, m_1,\ldots, m_{n-1}$ 
all   pointing to the east,
positioned on the $x$-axis ($\pos(m_i) = (i,0) \in \Zset^2$),
and for $0\leq i\leq n-2$ all monomers in the same state $s_0(m_i) = \sigma \in \Nset^+$ 
and $s_0(m_{n-1}) = 0$. 
\end{definition}

\begin{remark}\label{rem:angle}
The initial monomer state $\sigma \geq 0$ dictates that each monomer  wishes to turn (have a rule applied) a total $\sigma$ times, i.e. be rotated through an angle of $\sigma \pi/3$. 
\end{remark}
\begin{remark}[target configuration]\label{rem:targetconfig}
For intuition, if there was no notion of blocking in the Turning Machine model, that is, if the model permitted self-intersecting configurations (which it does not), then the final configuration $c$ of the Turning Machine in Definition~\ref{def:line} is a straight line of monomers sitting along the ray that starts at the origin and is at an angle of $\sigma \frac{\pi}{3}$, i.e. at positions $(0,0),(0,-1),\ldots,(0,- (n-1) )$ and all pointing to the west.
We call $c$ the desired {\em target configuration} of the line rotation Turning Machine $\tml{n}{\sigma}$. 
Also, if there was no notion of blocking: expected time to completion would be fast, $O(\log n)$ (by a generalisation of the analysis used in the proof of Lemma~\ref{lem:pi/3 rotation}).
However, a model with no blocking would be rather uninteresting. 
\end{remark} 
\noindent Figure~\ref{fig:60} (left) illustrates Lemma~\ref{lem:pi/3 rotation} and  Appendix~\ref{app:sec:rot} contains its straightforward, yet instructive, proof.  
\begin{restatable}{lemma}{lempioverthree}\label{lem:pi/3 rotation}
For each $n\in\Nset$, the line-rotating Turning Machine  $\tml{n}{1}$ computes its target configuration, and does so in expected $O(\log n)$ time. 
\end{restatable}

\begin{figure}
\centering
\includegraphics[page=3]{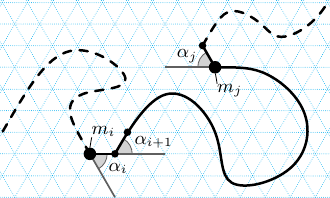}
\caption{Illustration of turn angle (Definition~\ref{def:oa}).  The turn angles $\alpha_i$ and $\alpha_{i+1}$ are positive (and to the left), and $\alpha_j$ is negative (and to the right).}
\label{fig:state-chain}
\end{figure}

\section{Tools for reasoning about  Turning Machines}\label{sec:tools}
\noindent The following lemma is illustrated in Figure~\ref{fig:60} (right).
\begin{lemma}\label{lem:180:li}
Let $n\in\Nset$ and let $T^{\leq 3}_{n}$ be a Turning Machine with initial states $0\le s(m_i)\le 3$ for all $0\leq i\leq n-1$. 
For any monomer $m_i$ in a reachable configuration $c$ of $T^{\leq 3}_{n}$, the monomers $\head(m_i)$ are positioned on or above $\ell_i$, and $\tail(m_i)$ are positioned on or below $\ell_i$, where $\ell_i$ is a horizontal line passing through $m_i$.
\end{lemma}
\begin{proof}
The claim follows from the fact that in any configuration of $\tml{n}{\leq 3}$, and for any $j\in\{0,1,\ldots,n-2\}$  
 the angle of the vector $\overrightarrow{ \pos(m_j) \pos(m_{j+1}) }$ (from monomer $m_j$ to $m_{i+1}$) is either $0^\circ$, 
$60^\circ$, $120^\circ$, or $180^\circ$ (and, in particular, is not  strictly between $180^\circ$ and  $360^\circ$).
\end{proof}

The notion of turn angle of a monomer is crucial to our analysis and is  illustrated in Figure~\ref{fig:state-chain}.\begin{definition}[turn angle]\label{def:oa}
Let $c$ be the configuration of an $n$-monomer Turning Machine and let $0\leq i < n-1$. 
The {\em turn angle $\alpha_i$ at monomer $m_i$} is the angle between $\overrightarrow{\pos(m_{i-1}) \pos(m_i)}$ and $\overrightarrow{\pos(m_{i}) \pos(m_{i+1})}$, and it is the positive counterclockwise angle if the points $\pos(m_{i-1}), \pos(m_i), \pos(m_{i+1})$ make a left turn\footnote{The notion of left or right turn along the three points $\pos(m_{i-1}), \pos(m_i), \pos(m_{i+1})$ can be formalised by considering the line $\ell_i$ running through $\pos(m_i)$, in the direction $\overrightarrow{\pos(m_{i-1}) \pos(m_i)}$, noting that $\ell_i$ cuts the plane in two, and defining the left- and right-hand side of the plane with respect to the vector along $\ell_i$.}, and the negative clockwise angle otherwise. 
\end{definition}

For a monomer $m_i$, the following definition gives a measure, $\Delta s(m_i)$, of how its state $s(m_i)$ has progressed since the initial configuration.  
\begin{definition}
Let $c$ be a reachable configuration of an $n$-monomer Turning Machine, 
let~$m_i$ be a monomer with state $s(m_i)$ in $c$  and  initial state  $s_0(m_i) \geq 0$. 
We define $\Delta s(m_i)$ to be the number of rule applications to (or, moves of) the monomer $m_i$ from the initial configuration to $c$.
That is, $\Delta s(m_i) = s_0(m_i) - s(m_i)$.
\end{definition}

\begin{lemma}[difference of State is $\leq 2$]\label{lem:s2} 
Let $n\in\Nset$, and let $c$ be any reachable configuration of an $n$-monomer Turning Machine $T_n$ with all monomers pointing in the same direction in its initial configuration, then
\[|\Delta s(m_i) - \Delta s(m_{i+1})| \leq  2\,,\] for all $0\le i<n-1$.
\end{lemma}
\begin{proof}
Let $m_k^t$, for $t\in\Nset$ and $k \in \{0,1,\ldots,n-1\}$, denote the $k^\mathrm{th}$ monomer in the $t^\mathrm{th}$ configuration~$c_t$.
Initially, $\Delta s(m^0_j)=0$ for all monomers $m_j$, and thus~${|\Delta s(m^0_i) - \Delta s(m^0_{i+1})|=0}$.

Observe, that $|\Delta s(m_i) - \Delta s(m_{i+1})| \neq 3$ because otherwise $\pos(m_i) = \pos(m_{i+2})$ making~$c$ a self-intersecting (non-simple) configuration, contradicting its definition.

By Equation~\ref{eq:r}, when a rule is applied to one of $m_i^t$ or $m_{i+1}^t$ the absolute value of its state decreases by~$1$ and its $\Delta s(\cdot)$ increases by~$1$.
Then $|\Delta s(m^t_i) - \Delta s(m^t_{i+1})| = |\Delta s(m^{t-1}_i) - \Delta s(m^{t-1}_{i+1})| \pm 1$.
When a rule is applied to  some other monomer $m_k$ with $i\neq k\neq j$, then $|\Delta s(m^t_i) - \Delta s(m^t_{i+1})| = |\Delta s(m^{t-1}_i) - \Delta s(m^{t-1}_{i+1})| \pm 0$.
Thus, after each rule application the value of $|\Delta s(m_i) - \Delta s(m_{i+1})|$ changes by at most $1$, and as it cannot be equal to $3$, we have that $|\Delta s(m_i) - \Delta s(m_{i+1})|\le 2$.
\end{proof}
We can now prove the following lemma, which gives a relation between the states of any two monomers of a Turning Machine and the geometry of the current configuration.
\begin{lemma}\label{lem:state-chain-general}
Let $c$ be any reachable configuration of an $n$-monomer Turning Machine $T_n$, whose initial configuration $c_0$ has all monomers pointing in the same direction, 
and let $m_i$ and $m_j$ be two monomers of $c$ such that $i<j<n-1$, then
\[
\Delta s(m_j)-\Delta s(m_i)
=\frac{3}{\pi}\sum_{k=i+1}^{j}\alpha_k\,,
\]
where $\alpha_k$ is the turn angle at monomer $m_k$.
\end{lemma}
\begin{proof}
For any intermediate configuration, the turn angle $\alpha_{i+1}$ between monomers $m_i$ and $m_{i+1}$ depends only on the number of moves each monomer has made.
Initially, $\alpha_{i+1}=0$.
It increases by $\pi/3$ each time monomer $m_i$ moves anti-clockwise or monomer $m_{i+1}$ moves clockwise, and it decreases by $\pi/3$ every time monomer $m_i$ moves clockwise or monomer $m_{i+1}$ moves anti-clockwise.
By Lemma~\ref{lem:s2}, for two consecutive monomers $m_i$ and $m_{i+1}$, in any configuration, $|\Delta s(m_i) - \Delta s(m_{i+1})|\le 2$.
Hence, for a pair of consecutive monomers $m_i$ and $m_{i+1}$, the turn angle $\alpha_{i+1}$ is in the range $[-2\frac{\pi}{3}, 2\frac{\pi}{3}]$,
and thus $\alpha_{i+1}=\frac{\pi}{3}(\Delta s(m_{i+1}) - \Delta s(m_{i}))$.
Summing over all $i$ gives the lemma conclusion. 
\end{proof}
The following technical lemma is used extensively for our main results. Intuitively, it tells us that high-state monomers are not blocked.
\begin{lemma}\label{lem:unblocked-states}
Let $T^{\le 5}_{n}$ be a Turning Machine with initial state values $0\le s(m_i)\le 5$ for all $0 \le i < n$. 
In any reachable configuration $c$ of $T^{\le 5}_{n}$ no monomer $m_i$ with $\Delta s(m_i) \le 1$ is blocked (neither temporarily blocked nor permanently blocked).
\end{lemma}
\begin{proof}
\begin{figure}[t]
\centering
\includegraphics[page=7]{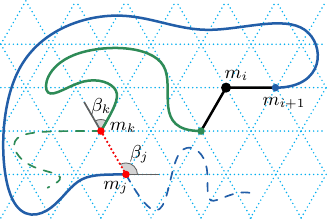}
\hfill
\includegraphics[page=8]{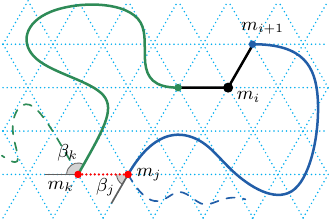}
\caption{Illustration for Lemma~\ref{lem:unblocked-states}.
Monomer $m_i$ is shown in black,  
$\head(m_i)$ is shown in blue and 
$\tail(m_i)$ is shown as the green curve plus the black monomer $m_i$.
Left: monomer $m_i$ is in its initial state ($\Delta s(m_i)=0$), and polygon $P$ is traversed anti-clockwise. 
Right: monomer $m_i$ has moved once ($\Delta s(m_i)=1$), and polygon $P$ is traversed clockwise.}
\label{fig:unblocked-states}
\end{figure}

Suppose, for the sake of contradiction, there is a blocked monomer $m_i$ with $\Delta s(m_i) \le~1$.
Then there exist two monomers $m_j\in \head(m_i)$ and $m_k \in \tail(m_i)$ such that $\pos(m_k)=\pos'(m_j)$, where $\pos'(m_j)$ is the position of $m_j$ in $\headp(m_i)$ (see Figure~\ref{fig:unblocked-states}). 

By definition of head and tail we know that $k\leq i < j$.
Consider the closed chain $P=\pos(m_k),\pos(m_{k+1}),\dots,\pos(m_{j-1}),\pos(m_{j}),\pos(m_k)$.
Since configurations are simple, $P$ defines a simple polygon.
The turn angles of a simple polygon sum to $2\pi$ if the polygon is traversed anticlockwise (interior of $P$ is on the left-hand side while traversing), or $-2\pi$ if the polygon is traversed clockwise (interior of $P$ is on the right-hand side). 
For $P$, this sum is defined as: 
\[
\alpha_P = \sum_{\ell=k+1}^{j-1}\alpha_{\ell} + \beta_j + \beta_k  = \pm 2\pi\,,
\]
where $\alpha_\ell$ is the turn angle at monomer $m_\ell$, and $\beta_j$ and $\beta_k$ are the turn angles of the polygon at vertices $\pos(m_j)$ and $\pos(m_k)$ respectively (see Figure~\ref{fig:unblocked-states}).
More precisely,
\[
\begin{aligned}
\alpha_\ell & = \angle(\overrightarrow{\pos(m_\ell)}-\overrightarrow{\pos(m_{\ell-1})},\overrightarrow{\pos(m_{\ell+1})}-\overrightarrow{\pos(m_\ell)})\,,\\
\beta_j & =\angle(\overrightarrow{\pos(m_j)}-\overrightarrow{\pos(m_{j-1})},\overrightarrow{\pos(m_k)}-\overrightarrow{\pos(m_j)})\,, \text{ and}\\
\beta_k & =\angle(\overrightarrow{\pos(m_k)}-\overrightarrow{\pos(m_j)},\overrightarrow{\pos(m_{k+1})}-\overrightarrow{\pos(m_k)})\,.
\end{aligned}
\]
Furthermore, by Lemma~\ref{lem:state-chain-general},
\[
\Delta s(m_{j-1})-\Delta s(m_k) = \frac{3}{\pi}\sum_{\ell=k+1}^{j-1}\alpha_{\ell}\,.
\]
Thus,
\[
\begin{split}
\Delta s(m_{j-1}) = \Delta s(m_k) + \frac{3}{\pi}\sum_{\ell=k+1}^{j-1}\alpha_{\ell} = \Delta s(m_k) + \frac{3}{\pi}(\pm 2\pi -\beta_j - \beta_k)\\ = \Delta s(m_k) \pm 6 - \frac{3}{\pi}(\beta_j + \beta_k) \,.
\end{split}
\]
Observe that when a monomer $m_i$ moves, its head translates in the direction corresponding to the current direction of $m_i$ rotated by angle $2\pi/3$.
Therefore, the state of $m_k$ can be represented as a function of the state of $m_i$ and the angle $\beta_k$, more precisely 
\[
\Delta s(m_k) = \Delta s(m_i) + 2 + \frac{3}{\pi}\beta_k \,.
\]
(See Figure~\ref{fig:unblocked-states} for an example.) Therefore, by the previous two equalities 
\[
\Delta s(m_{j-1}) = \Delta s(m_i) + 2 \pm 6 - \frac{3}{\pi}\beta_j \,.
\]
Recall, that the angle $\beta_j \in [-2\pi/3,2\pi/3]$, that $0 \leq \Delta s(m_i) \leq 1$ by the assumption of the lemma, and that $\Delta s(m_{j-1}) \leq s$.
If the polygon defined by $P$ is traversed anti-clockwise, then
\[
\Delta s(m_{j-1}) = \Delta s(m_i) + 8 - \frac{3}{\pi}\beta_j \ge 0 + 8 - 2 = 6\,,
\]
which implies that $s(m_{j-1})$ is out of the range of valid states, as $m_{j-1}$ must have moved more times as its initial state.
Else, if the polygon $P$ is traversed clockwise, then
\[
\Delta s(m_{j-1}) = \Delta s(m_i) - 4 - \frac{3}{\pi}\beta_j \le 1 - 4 + 2 = - 1\,,
\]
which again implies that $s(m_{j-1})$ is out of the range of valid states, as $m_{j-1}$ must have moved in the wrong direction.
In either case we contradict that the state $s(m_{j-1})$ is in the range of valid states, and, therefore, the monomer $m_i$ is not blocked.
\end{proof}

\begin{lemma}\label{lem:pred-TM}
Let $\tml{n}{s}$ be a line-rotating Turning Machine  with $s\le 5$. 
Let $c$ be a reachable configuration of $\tml{n}{s}$ where each monomer $m_i$ in $c$ has $s_c(m_i) < s$. 
Then the line-rotating Turning Machine $\tml{n}{s-1}$
has a reachable configuration $c'$ 
 such that for every $m_i$, $s_{c'}(m_i)=s_{c}(m_i)$ and 
 the geometry (chain of positions) of $c$ is equal to that of the rotation of $c'$ by $\pi/3$ around the origin.
\end{lemma}
\begin{proof} 
Consider the sequence $\rho_c$ rule applications (moves) that brings the initial configuration of $\tml{n}{s}$ to configuration $c$.
We claim that $\rho_c$ can be converted into another sequence $\rho_{c'}$, of the same length, in which the first $n-1$ moves are by monomers in state $s$.

First, we claim: 
for any two consecutive moves, where the second move is applied to a monomer in state $s$, swapping the two moves results in a valid sequence of moves transforming the Turning Machine into the same configuration.
Let the first move be applied to monomer $m_i$ which transitions from state $s'$ to $s'-1$, and the second move be applied to monomer $m_j$ which transitions from state $s$ to $s-1$.
Suppose for the sake of contradiction that swapping the moves results in at least one of the monomers $m_i$ or $m_j$ being blocked.
We begin by attempting to apply the move to monomer $m_j$, but, by Lemma~\ref{lem:unblocked-states}, that move is not blocked.
Then we attempt to apply a move to monomer $m_i$, but that is not blocked either since the coordinates of all monomers before and after swapping the two moves are exactly the same; i.e. the resulting configuration is a valid (non-self-intersecting) configuration in both cases. 
Hence neither monomer is blocked.

Thus, the original sequence of moves resulting in configuration $c$, can be converted into another sequence where the first $n-1$ moves are applied to monomers in state $s$. 
Then, after the first $n-1$ moves the configuration of $\tml{n}{s}$ is equivalent to the initial configuration of $\tml{n}{s}$ but rotated by $\pi/3$ and with all monomers in state $s-1$. Hence equivalent to the initial configuration of $\tml{n}{s-1}$ rotated by $\pi/3$. 

Applying the remaining moves to $\tml{n}{s-1}$  transforms it into configuration~$c'$. 
\end{proof}

\section{Line rotation: \texorpdfstring{$5\pi/3$}{5π/3} possible and fast, \texorpdfstring{$2\pi$}{2π} impossible}\label{sec:lines}
\subsection{Line rotation to \texorpdfstring{$5\pi/3$}{5π/3}}
In this section we show that for $1\leq s \leq 5$ the line-rotation Turning Machine $\tml{n}{s}$ computes its target configuration of a $s\pi/3$ rotated line (Theorem~\ref{thm:line-rotation-computes}), 
and does so in expected time $O(\log n)$ (Theorem~\ref{thm:s5fast}).
In addition to those results for any state $s \leq 5$, in Appendix~\ref{app:sec:rot} we include stand-alone proofs for each of $s=1$, $s=3$,  and $s=4$ which showcase a variety of geometric techniques for analysing Turning Machine movement, but are not needed to prove our main results. 
Also, the cases of $s=1$ and $s=3$ are illustrated in Figures~\ref{fig:60} and~\ref{fig:pi}.

\begin{figure}[t]
  \centering
    \includegraphics[page=4]{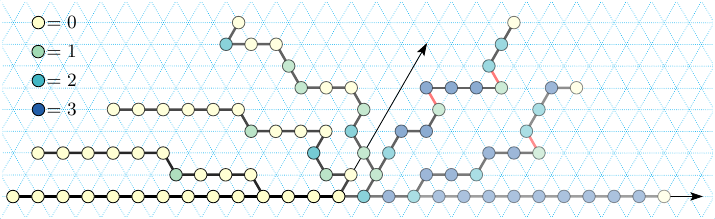}
  \caption{Example trajectory of the Turning Machine $\tml{n}{3}$  that rotates a line of east-pointing monomers by an angle of $\pi$. Illustration for 
  Theorem~\ref{thm:s5fast} with $s=3$
  (and for Lemma~\ref{lem:pi:blocking} and Theorem~\ref{thm:pi rotation} in Appendix~\ref{app:sec:rot}). 
  Seven configurations are shown, the initial configuration has all monomers in state 3 (blue), final in state 0 (yellow). 
  Darker shading indicates later in time. 
  A red bond (edge) indicates a blocked monomer. 
  The proof of Lemma~\ref{lem:pi:blocking} shows that only monomers in state 1 are ever blocked and only when they are adjacent to a monomer in state~3, and that all such blockings are temporary -- if we wait long enough they become unblocked.
  }
    \label{fig:pi}
\end{figure}

\newcommand{\tm}[2]{T_{#1}^{\leq #2}}

\begin{theorem}\label{thm:line-rotation-computes}
For each $n\in\Nset$ and $1 \le s\le 5$, the line-rotation Turning Machine $\tml{n}{s}$ computes its target configuration.
\end{theorem}
\begin{proof}
We prove by induction on $1\leq s \leq 5$ that any reachable configuration $c$ of $\tml{n}{s}$ is not permanently blocked.

Base case $s=1$. In any configuration reachable by $\tml{n}{1}$, monomers have either state $s=1$ or $0$. Monomers in state $s=1$ cannot be permanently blocked by Lemma~\ref{lem:unblocked-states}. Thus, any non-final configuration is not permanently blocked.

Assume for $s-1$ the claim is true, i.e.\ it holds for $\tml{n}{s-1}$. We will prove that for $s$ it is also true, i.e.\ it holds for $\tml{n}{s}$.
Suppose, for the sake of contradiction, there is a permanently blocked configuration $c$ of  $\tml{n}{s}$ for some $n\in\Nset$ and $s \leq 5$.
If there is no monomer in $c$ in state $s$, then by Lemma~\ref{lem:pred-TM} there exists a corresponding configuration $c'$ in $\tml{n}{s-1}$ with monomers $m'_0,m'_1,\dots,m'_{n-1}$, such that, for any monomer $m_i$ in $c$ with state $s_i < s$ the corresponding monomer $m'_i$ in $c'$ has the same state $s_i$.
Configurations $c$ and $c'$ form chains equal up to rotation by angle $\pi/3$.
Configuration $c'$ is not blocked by the induction hypothesis, thus configuration $c$ cannot be blocked either.

On the other hand, if there is a monomer $m_i$ in configuration $c$ in state $s$, then by Lemma~\ref{lem:unblocked-states} it is unblocked, and configuration $c$, again, is not blocked.

Hence the induction hypothesis holds for $s$, and $\tml{n}{s}$ does not have a reachable permanently blocked configuration.
\end{proof}

\begin{theorem}\label{thm:s5fast}
For each $n\in\Nset$ and $1 \le s\le 5$, the line-rotation Turning Machine $\tml{n}{s}$ computes its target configuration in expected  time $O(\log n)$.
\end{theorem}
\begin{proof}
By Theorem~\ref{thm:line-rotation-computes}, $\tml{n}{s}$ computes its target configuration. 
For the time analysis we use a proof by induction on $ u \in \{ 0,1,\ldots, s\}$, in decreasing order.

The induction hypothesis is that for a reachable configuration $c_u$ of $\tml{n}{s}$ with maximum state value $u$ (there may be states $< u$ in the configuration), the expected time to reach a configuration $c_{u-1}$ with maximum state $u-1$ is $O(\log n)$.

For the base case we let $u=s$ and assume $c$ is such that all monomers are in state $u$. 
Hence $c$ is an initial configuration and hence, by definition, is reachable.
By Lemma~\ref{lem:unblocked-states}, monomers in state $s$ are never blocked and 
hence we claim that the first configuration with maximum state $u-1$ appears after expected time  $O(\log n)$.
To see this claim, note that for each monomer $m_i$  in state $s(m_i)=u$  the rule application that sends $m_i$ to state $u-1$ occurs 
at rate 1, independently of the states and positions of the other monomers (by Lemma~\ref{lem:unblocked-states}, there is no blocking of a monomer in state $u=s$).  Since there are $n$ monomers in state $u$, the expected time for all $n$ to transition to $u-1$ is~\cite{graham1989concrete}:
\begin{equation}\label{eq:et}
 \sum_{k=1}^{n} \frac{1}{k}  = O(\log n)  \, .
\end{equation}

We assume the inductive hypothesis is true for $0< u+1 \le s$, and we will prove it holds for~$u$.
Thus, there exists a reachable configuration $c_u$ where the maximum state value is $u \leq s$, which is reachable from $c_{u+1}$ in expected $O(\log n)$ time.
Let there be $n' \leq n$ monomers in state $u$ in $c_u$.
By Lemma~\ref{lem:pred-TM}, there is a line-rotating Turning Machine $\tml{n}{u}$ that has a reachable configuration $c'_u$  such that for every $m_i$ in $c_u$, $s_{c'_u}(m_i)=s_{c_u}(m_i)$ and the positioning of $c_u$ is equal to the rotation of $c'_u$ by $\pi/3$ around the origin.
By Lemma~\ref{lem:unblocked-states} monomers in state $u$ in $\tml{n}{u}$ are never blocked, hence monomers in state $u$ in $c_u$ are not blocked either.  
Setting $n=n'$ in Equation~\eqref{eq:et}, and noting that $O(\log n')=O(\log n)$, proves the inductive hypothesis for $u$.

Since we need to apply the inductive argument at most $s\leq 5$ times, by linearity of expectation, the expected finishing time for the $s$ processes is their sum,  $5 \cdot O(\log n) = O(\log n)$. 
\end{proof}

\subsection{Negative result: Line rotation to \texorpdfstring{$2\pi$}{2π} is impossible}\label{sec:negative line rotation}
\begin{figure}[t]
\centering
\includegraphics[page=14,width=0.75\textwidth]{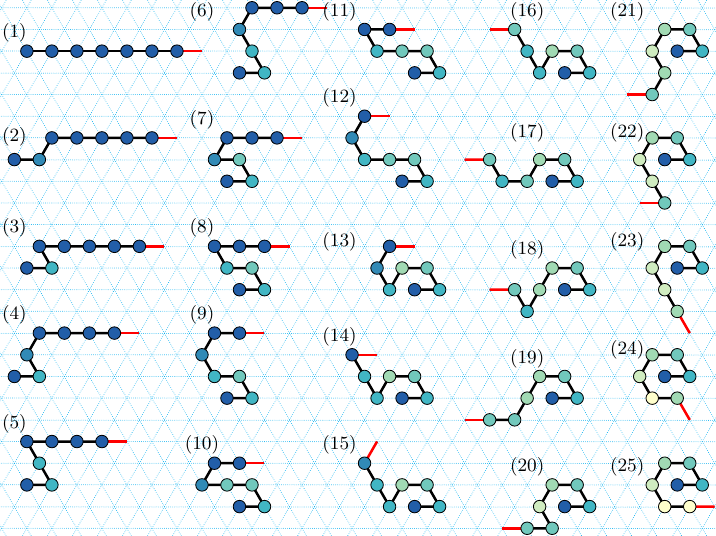} 
\caption{Impossibility of $360^\circ$ line rotation (Theorem~\ref{thm:no 2pi rotation}), by showing that for all $n\in\Nset$, the line-rotation Turning Machine $\tml{n}{6}$ has a reachable but permanently blocked configuration.   
Looking at the evolution of the first seven monomers (i.e.~ignore the rotation of the red line segment) 
we see one trajectory of the Turning Machine that exhibits permanent blocking in the final (bottom-right) configuration, which has respective states 6,4,3,2,1,0,0.
We imagine the red line segment as representing an arbitrary long sequence of monomers running collinear with it, and transitions 14--16, 22--23, and 24--25, each representing the (many step) rotation of the red line by respective angles of 180$^\circ$, 120$^\circ$ and $60^\circ$.
These rotations of the red line can proceed by two applications of Theorem~\ref{thm:line-rotation-computes} (first with $s=3$, then with $s=1$) and the fact that the first monomer of the red line is strictly above, or below, the first seven monomers. Hence the final, permanently blocked, configuration is reachable no matter what length the red line~is.}
\label{fig:T7}
\end{figure}
\begin{theorem}\label{thm:no 2pi rotation}
For all $n\in\Nset, n \geq 7$, the line-rotating Turning Machine $\tml{n}{6}$ does not compute its target configuration. 
In other words, there is a permanently blocked  reachable configuration. 
\end{theorem}
\begin{proof}
Figure~\ref{fig:T7}, looking only at the first 7 (initially blue) monomers, shows a valid trajectory of $\tml{7}{6}$ , then ends in a permanently blocked configuration, hence the lemma holds for $n=7$.

Let $n>7$, and in Figure~\ref{fig:T7} let the red line segment denote a straight line $\ell_{n-7}$ of $n-7$ monomers co-linear with the red line segment. 
By inspection, it  can be verified that (a)~in all 25 configurations the line $\ell$ does not intersect any blue monomer, and moreover
(b)~the transitions from configurations  
 1 through 14, configurations 17 through 23, and configuration 24 to 25 are all valid, meaning that the length $n-7$ line $\ell_{n-7}$ does not block the transition. 
The transitions for configurations 14 through 17 are valid by Theorem~\ref{thm:line-rotation-computes} (with $s=3$)  and the fact that the last blue monomer (the origin of $\ell_{n-7}$) is strictly  above all other blue monomers (hence the $180^\circ$ rotation of $\ell_{n-7}$ proceeds without permanent blocking by blue monomers). 
The transition for configuration 23 to 24 is valid by applying Lemma~\ref{lem:pi/3 rotation} (or Theorem~\ref{thm:line-rotation-computes}, with $s=1$)
reflected through a horizontal line that runs through  the last blue monomer, 
and the fact that the last blue monomer (the origin of $\ell_{n-7}$) is strictly below all other blue monomers (hence the $60^\circ$ rotation of $\ell_{n-7}$ proceeds without permanent blocking). 
Thus all transitions are valid and the permanently blocked configuration is reachable, giving the lemma statement.  
\end{proof}

\section{Folding shapes}\label{sec:folding shapes}\label{sec:shapes}

In this section we show how to fold certain kinds of shapes with Turning Machines, and we show that other kinds of shapes are impossible to fold, and that
others are not foldable but can be approximated. 

\begin{definition}[shape]\label{def:shape}
A \emph{shape} is a set of points in the grid $G_\triangle$ such that its induced graph is connected (using unit length edges on  $G_\triangle$).
\end{definition}

\begin{definition}[$xy$-connected shape]\label{def:xy-shape}
An \emph{$xy$-connected shape} $S$ is a set of points in the grid $G_\triangle$ such that if we remove all the edges parallel to the direction $\vec{w}$ from the induced graph of $S$, it remains connected.
\end{definition}

For example squares are a classic benchmark shape in self-assembly: 
\begin{definition}[$n \times n$  square]\label{def:square}
For $n\in\Nset$ the $n \times n$ square is the set of  points $(x,y) \in \Nset^2$ such that $0 \leq x,y < n$.
\end{definition}

Typically, we are interested in folding a shape with a Turning Machine, and ideally we'd like each monomer to sit on one point of the shape, but in the non-ideal case there are errors:  

\begin{definition}[error of folding a shape]\label{def:error}
Given a shape $S$ and Turning Machine configuration $c$,
the {\em error of folding} of $c$ with respect to $S$ is the size of the symmetric difference of the shape and the positions of $c$ 
(i.e., the number of positions of $c$ that are not in $S$, plus the number of positions $S$ not in $c$). 
\end{definition}

\subsection{Folding zig-zag paths and $n \times n $ squares}\label{sec:zz}

Recall that a path in $\mathbb{Z}^2$ is simple, connected and directed, and its length is the number of its points. Intuitively, a zig-zag path is a path that winds over and back parallel to the $x$ axis, while heading off in one direction (positive or negative) along the other two axes. 
For example, Figure~\ref{fig:zz} shows a positive zig-zag path that rasters an $8\times8$ square (traced out by a Turning Machine configuration). 

In this section, we apply our techniques from previous sections to help us prove Theorem~\ref{thm:zz}, which states that zig-zag paths (Definition~\ref{def:zz}) are  foldable by Turning Machines in expected time logarithmic in path length. 
Since for any $n \in \mathbb{N}$, an $n \times n$ square can be rastered (traced out) using a zig-zag path, we then get that they too are efficiently foldable by Turning Machines.

\begin{definition}[Zig-zag path]\label{def:zz} 
A {\em positive zig-zag path}, or simply {\em zig-zag path}, is a path in $\mathbb{Z}^2$ composed of unit length line segments that run along directions $\pm\vx$, $+\vy$ and $+\vw$. 
A {\em negative zig-zag path}, has unit length line segments that run  along 
$\pm\vx$, $-\vy$ and $-\vw$. 
\end{definition}
The previous definition captures the intuition of a path that zig and zags over and back (e.g. Figure~\ref{fig:zz}) but is much more general since a positive, respectively negative, zig-zag path is any path that is monotone in $y$ (and hence in $w$), respectively  $-y$ (and hence in $-w$).

\newcommand{\dir}{\textrm{direction}}
\newcommand{\sL}{restricted-$L^3$-system\xspace}

The following lemma will find future use in showing that a Turning Machine folds a positive zig-zag path. 
\begin{lemma}\label{lem:zz}
Let $c$ be any configuration reachable by a line-rotating Turning Machine~$L^3_n$, then there is a Turning Machine $T_n$ that has $c$ as its target configuration. Moreover,  $T_n$ runs in expected time $O(\log n)$.
\end{lemma}

\begin{proof}
We first define $T_n$ as having the initial configuration $c_0$, of length $n\in\mathbb{N}$, where for $0 \leq i \leq n-1$ monomer $m_i$ in $c_0$ has state $3-s_{i,c}$, where $s_{i,c}$ is the state of monomer $m_i$ in $c$. 
By hypothesis, $c$ is reachable in $L^3_n$, which means there is a trajectory (sequence of rule applications) from the initial configuration of $L^3_n$ to $c$, thus  applying the same sequence of moves in $T_n$, starting with $c_0$, also yields~$c$. 
We need to show that $c$ is the target of $T_n$ (meaning all trajectories from  $c_0$ reach $c$).

We will define a ``suppressed-$L^3$-system''  to be a Turning Machine-like system\footnote{We are using a technique here that leverages the statement of Lemma~\ref{lem:pi:blocking} from the appendix. We could alternatively use the same technique as Lemma~\ref{lem:pi:blocking}, which more directly reasons about rule applications and blocking, without needing to define the \sL.} that acts like $L^3_n$ in all ways, except that for each $i$,
after monomer $m_i$ reaches state $3-s_{i,c}$ any rule application to $m_i$ is ``suppressed'' (not applied).\footnote{This notion of ``suppression'' should not be confused with blocking, for the sake of analysis, we are simply defining a new model where we get to arbitrarily choose which rules are suppressed.}  
By Lemma~\ref{lem:unblocked-states}, for any~$i$ where $s_{i,c} \in \{ 1,2,3 \}$, monomer $m_i$ in the suppressed-$L^3$-system is never blocked since to experience blocking a monomer needs be in the process of transitioning from state 1 to 0, which never happens for  $s_{i,c} \in \{ 1,2,3 \}$).
In other words, we only expect blocking to affect monomers $m_i$ for which $s_{i,c}=0$.  
Consider such a monomer $m_i$.  
By applying Conclusion~\ref{concl:only state 1} of Lemma~\ref{lem:pi:blocking} to $m_i$ we note that blocking only occurs if $m_i$ is in state~1 and is adjacent to a monomer in state~3. 
However, such blocking is temporary, which follows from the claim that that adjacent monomer~$m_j$ eventually leaves state 3: to see the claim note that: 
(a)~$L^3_n$ does not experience permanent blocking (hence $m_j$ does not) and (b)~$s_{j,c} < 3$, since otherwise two monomers would occupy the position in $c$ which contradicts the positions of $c$ tracing out a path (i.e.~non-self-intersecting).
Therefore, since there is no permanent blocking in the suppressed-$L^3$-system it always reaches $c$. 

Both $T_n$ and the \sL have the same set of trajectories because (a)  suppressed rules in the \sL make no change to a configuration, and (b) for any configuration $c'$ it has the same set of applicable (non-suppressed) rules in the \sL and in $T_n$.

For the time analysis, we analyse the expected time for the \sL to reach $c$. 
To do that, we consider two processes (two Markov chains): 
the process where rules are applied/blocked and the process where rules are  suppressed. Observe that suppressed rules do not change a configuration and do not change the rate of applications of  rules, thus from now on in our analysis, we ignore the latter process. 
For the \sL to reach $c$, applicable rules are either unblocked (rate 1 per rule application) or temporarily blocked, and the latter case only occurs when a monomer $m_i$ in state 1 is adjacent to a monomer $m_j \in \{m_{i-1}, m_{i+1} \} $ where $m_j$ is in state 3 and is unblocked (as argued above). Since all such $m_j$ are unblocked, the expected time until they are all in a state $< 3$ is $O(\log n)$, after which there are no blocked monomers in the resulting configuration. The rule application process then completes in expected time $O(\log n)$, giving an overall expected time of $O(\log n)$.
\end{proof}

\begin{theorem}\label{thm:zz}
A positive, or negative, zig-zag path $P$, of length $n$, is foldable by a Turning Machine in expected time $O(\log n)$.
\end{theorem}
\begin{proof}
Let $P$ be a positive zig-zag path. 
We begin by defining a length-$n$ Turning Machine $\tzz$ whose {\em target}\footnote{As discussed in Remark~\ref{rem:targetconfig}, a {\em target} configuration is one where all states are 0, in other words the intended final configuration of the Turning Machine. Note that we've not yet proven that such a target is always reachable---that's our goal in fact.} configuration $c$ traces out the zig-zag path $P$.
By the definition of $P$, each monomer in $c$ points in one of the directions $\pm \vec x, +\vec y$ and $+ \vec w$. 
For any configuration $c'$, and $0\leq i \leq n-1$, let $m_{i,c'}$ denote monomer $m_i$ in configuration $c'$.   
In the initial configuration $c_0$ of $\tzz$, for $i\in\{0,1,\ldots,n-2 \}$ we define monomer $m_{i,c_0}$ to have the state
\begin{equation}\label{eqn:zz states}
s(m_{i,c_0}) =
\begin{cases}
0 & \text{ if } \dir(m_{i,c})=+\vx \\
1 & \text{ if } \dir(m_{i,c})=+\vy \\
2 &\text{ if } \dir(m_{i,c})=+\vw\\
3 &\text{ if } \dir(m_{i,c})=-\vx
\end{cases}
\end{equation}
and $s(m_{n-1,c_0})=0$, 
where $\dir(m_i)= \pos (m_{i+1}) - \pos(m_i) $, i.e.  the direction of $m_i$ as defined in Section~\ref{sec:def}. 
Thus the initial configuration has all states $\leq 3$.

Let $s_i$ denote the initial state of monomer $m_i$. Consider the line-rotating Turning Machine~$L^3_n$ in its initial configuration and consider the following trajectory. We begin by moving sequentially along the chain of monomers from left to right starting at $m_0$, and for each monomer $m_i$ with $s_i>0$ we apply a turning rule to $m_i$ once. After we reach monomer $m_{n-1}$, we move back along the chain and for each monomer $m_i$ such that $s_i>1$ we apply a turning rule to it. After we get back to $m_0$, we once again move left to right and for each $m_i$ with $s_i>2$ we apply a turning rule to it. Once we reach $m_{n-1}$ again we are done. At this point each monomer $m_i$ has had $s_i$ turning rules applied and thus we are in a configuration which traces out the path $P$. Since target configuration $c$ of $\tzz$ traces out the path $P$, $c$ is reachable by $L^3_n$. Thus by Lemma~\ref{lem:zz}, $\tzz$ folds $P$ in $O(\log n)$ expected time. 

For a negative zig-zag path $P$, we use the same proof but mirror-flipped around the $x$-axis, and using negative (instead of positive) turning numbers.
\end{proof}

An example of an $n\times n$ square
$n\in\Nset$ (Definition~\ref{def:square}), with $n=8$ in the example, is traced out by the path in Figure~\ref{fig:zz} (left). 
Since for any $n$ the $n\times n$ square is traced out by a positive zig-zag path,  by Theorem~\ref{thm:zz} we immediately get the following corollary:

\begin{corollary}\label{cor:square}
For any $n\in\Nset$, the $n \times n$  square is foldable with zero error and in expected time $O(\log n)$.
\end{corollary}

\begin{figure}[t]
\centering
\includegraphics[page=11,width=0.47\textwidth]{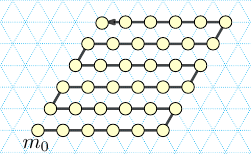}
\hfill
\includegraphics[page=12,width=0.5\textwidth]{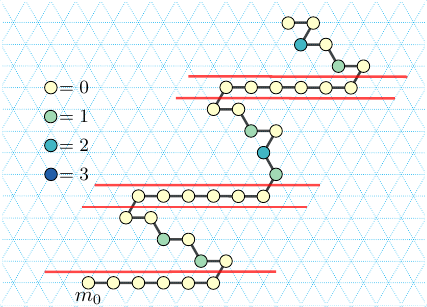}
\caption{Left: The yellow points define an $n\times n$ zig-zag square for $n=8$ (Definition~\ref{def:square}).
Corollary~\ref{cor:square} states that, for any $n$, such a square is foldable, and the proof works by showing there is a Turning Machine $\tzz$ that traces the positive zig-zag path shown (with monomer $m_0$ at the origin).
Right: An intermediate configuration of \tzz:  all initially-non-0-monomers have moved. The horizontal lines (in red) subdivide $\tzz_n$ into $n$ independent subchains each equivalent to a separate line-rotating Turning Machine $\tml{n}{3}$, which gives some intuition that the target configuration is safely reached on all trajectories without permanent blocking (although our actual proof proceeds by a different argument).}
\label{fig:zz}
\end{figure}

\subsection{Folding \texorpdfstring{$y$}{y}-monotone shapes, but with error}\label{sec:y-monotone}

If we permit error in a folding, by which we mean that some points of the folding are not in the shape (or vice-versa) then any shape is foldable. In particular, for any shape $S$ there is a (boring) Turning Machine that folds a zig-zag traversal of the bounding rectangle of $S$. The error will be bounded by the area of $S$. Can we do any better? 
In this section, we show that, yes, the class of $y$-monotone shapes are foldable with error bounded merely by the perimeter of $S$, and  if we allow spatial scaling for $y$-monotone shapes with the $yw$-separator property (see Definition~\ref{def:yw-separator}) we have zero error.

\begin{definition}
We say that a shape is  $y$-monotone if the points along each of its y-coordinates form a line segment. 
\end{definition}

The {\em perimeter} of a shape is defined as the set of points in the shape that are adjacent (in $\Zset^2$) to point(s) not in the shape, and the  perimeter length is the number of points in the perimeter. 
A {\em traversal} is a simple path in $\Zset^2$.  
A {\em zig-zag traversal} is a traversal that traces a positive (or negative) zig-zag path (Definition~\ref{def:zz}) such as the $8 \times 8$ square shown in Figure~\ref{fig:zz} (Left).

 \begin{theorem}\label{thm:y-monotone-scale1} 
 Any y-monotone shape $S$ can be folded with error no more than the perimeter length of $S$, and no more than the perimeter of the shape induced by the folding. Moreover, $S$ is folded with error in expected time $O(\log n)$. 
 \end{theorem}
\begin{proof}

We will give a zig-zag traversal that covers all points of the shape, as well as possibly a number of points outside the shape, but within the stated error bound. 
We will then show that the traversal is foldable using the techniques from Section~\ref{sec:zz}.

For each $y$-coordinate $y_i$  in the shape, i.e. $y_0 < y_1 < \cdots < y_{H-1} $ where the shape is of height (or span) $H$ along the y-axis, let  $\sigma_i$ denote the set of points along~$y$.  The set $\sigma_i$ is a line segment by y-monotonicity of the shape $S$.   

We define a zig-zag traversal $\Sigma$ that includes all points of $S$:  If $i$ is an even index such that $y_i$ is a $y$-coordinate of $S$, then let $\Sigma_i$ include all points of $\sigma_i$, plus any points $p$ to the right of $\sigma_i$ such that there is a point $p+(0,1) \in \sigma_{i+1}$, and any points $p$ to the left of  $\sigma_i$ such that there is
$p+(0,-1) \in \sigma_{i-1}$. 
Else if $i$ is an odd index such that $y_i$ is a $y$-coordinate of $S$, then let $\Sigma_i$ include all points of $\sigma_i$, plus any points $p$ to the left of $\sigma_i$ such that there is a point $p+(0,1) \in \sigma_{i+1}$, and any points $p$ to the right of  $\sigma_i$ such that there is
$p - (0,1) \in \sigma_{i-1}$.

In either case $\Sigma_i$ is a line segment, this follows from the fact that $S$ is connected and y-monotone. 
For $j\in \{0,1,\ldots , |\Sigma_i|-1 \}$ we write $\Sigma_i(j)$ to mean the $j$th, where for even $i$ we index from the left hand side (from smallest to largest $x$-coordinate), and for odd $i$ we index from the right hand side   (from largest to smallest $x$-coordinate) of the line segment $\Sigma_i$. 

We next claim that the sequence $\Sigma$ is a path that includes all points of $S$. Writing $\Sigma$ point-by-point:
\begin{align*}
\Sigma =\,& \Sigma_0(0), \Sigma_0(1), \ldots, \Sigma_0(|\Sigma_0|-1),\\ 
& \Sigma_1(0), \Sigma_1(1), \ldots, \Sigma_1(|\Sigma_1|-1),\\
& \vdots \\
& \Sigma_{H-1}(0), \Sigma_{H-1}(1), \ldots, \Sigma_{H-1}(|\Sigma_{H-1}|-1)
\end{align*}
First, $\Sigma$ includes all points of $S$ because for each $i$ where $y_i$ is a $y$-coordinate of $S$, the set of points $\cup_i \Sigma_i$ includes all points at y-coordinate $y_i$ of $S$ (because $\cup_i \sigma_i$ does and $\cup_i \sigma_i \subset \cup_j \Sigma_j$).
Second, we need to show that $\Sigma$ is a simple path. 
$\Sigma$ is composed of $H-1$ non-intersecting line segments therefore is simple. To see that $\Sigma$ is connected it suffices to observe that for each $j\geq 0$, $\pos{(\Sigma_j(|\Sigma_j|-1))} = \pos{(\Sigma_{j+1}(0))} - (0,1)$.
This completes the claim that $\Sigma$ is a path that includes all of $S$. 

For the error, any point in $\Sigma$ that is not in $S$ is adjacent to a perimeter point of $S$, this follows from the definition of $\Sigma_i$, i.e. points in
$\Sigma_i \setminus \sigma_i$ are not in $S$ but are adjacent to points in $(\sigma_{i-1} \cup \sigma_{i+1}) \subset S$. 

Also, there is no perimeter point $p$ of $S$ such that there are $\geq2$ points of $\Sigma$  adjacent to $p$ (if there were, this would contradict $\Sigma$ being simple, or  $S$  being  connected/y-monotone).

Since $\Sigma$ is a zig-zag path, it is foldable in $O(\log n)$ expected time by Theorem~\ref{thm:zz}. 
\end{proof}

\noindent Theorem~\ref{thm:y-monotone-scale1} states that  all $y$-monotone shapes are foldable, albeit with some error.  
One might ask: Is it possible to fold $y$-monotone shapes without error?  
The answer is no, via a straightforward argument: 
\begin{theorem}\label{thm:cross} 
There are $y$-monotone shapes that require folding error $ > 0$. 
\end{theorem}
\begin{proof}
Let $S$ be a cross shape with width-1 line segments for arms, with arm length $>1$, as illustrated in Figure~\ref{fig:cross}. 

Since $S$ is $y$-monotone, by Theorem~\ref{thm:y-monotone-scale1} it is foldable (with error) by a Turning Machine, 
but since it is not Hamiltonian (does not have a traversal), any folding of it has non-zero error. 
\end{proof}
\begin{figure}[ht]
\floatbox[{\capbeside\thisfloatsetup{capbesideposition={left,bottom},capbesidewidth=5cm}}]{figure}[\FBwidth]
{\caption{A foldable shape, but one that requires error $>0$.}\label{fig:cross}}
{\includegraphics[width=3cm]{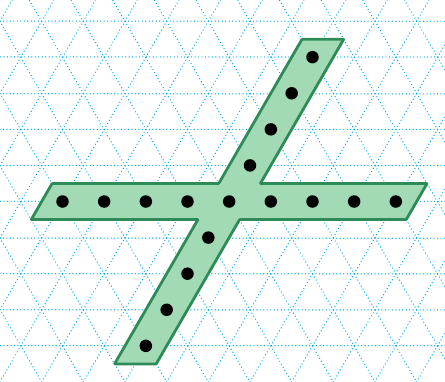}}
\end{figure}

Theorem~\ref{thm:cross} feels unsatisfactory: it is not fair to ask Turning Machines to exactly (zero error) fold shapes that don't even have an error-free traversal. Later, in Theorem~\ref{thm:spiral}, we give a much stronger result:
there is a class of shapes that are unfoldable (with zero error), even though each shape in the class has a (zero error) traversal, and moreover that traversal is foldable on at least one Turning Machine trajectory. 

We leave the following open problem as future work:

\begin{openproblem}\label{op:foldable shapes} 
Characterise the class of shapes that are foldable with 0 error.
\end{openproblem}
In the next section we make some partial progress.

\subsection{Folding scaled shapes with 0-error}\label{sec:exact-folding}

In this section we present an approach to folding {\em spatially scaled} shapes with $0$-error. In our setting, factor-2 scaling is sufficient. 
See an example in Figure~\ref{fig:yw-monotone}.

\begin{figure}[t]
\centering
\includegraphics[page=1,width=0.7\textwidth]{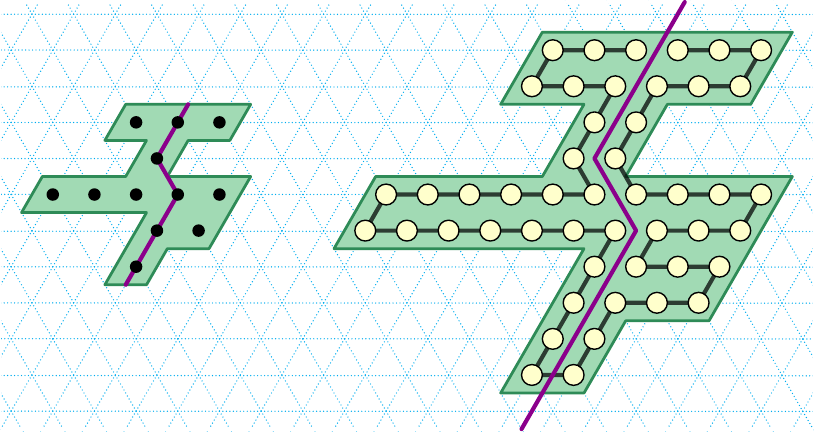}
\caption{Left: an $xy$-connected $y$-monotone shape $S$ with a $yw$-separator (purple). Right:~$S_{\times 2}$, a scaled version of $S$ by a factor two, its traversal $\Sigma$ (black) by a chain of a Turning Machine $T^{\Sigma}$, and the cut $\gamma$ (purple).}
\label{fig:yw-monotone}
\end{figure}

\begin{definition}
Given a shape $S$, we define its scaled version $S_{\times k}$ by a factor $k$ in the following way. Each vertex $(i,j)\in S$ is replaced by a $k\times k$ square in shape $S_{\times k}$. That is, for every $(i,j)\in S$, shape $S_{\times k}$ contains all grid-vertices $\{(ki+a,kj+b)\mid 0\le a,b<k\}$.
\end{definition}

If we blindly attempt to take our approach of folding with error, Theorem~\ref{thm:y-monotone-scale1}, and attempt to apply it directly to fold factor-$2$ scaled shapes,  we will fail (for example, consider a factor-$2$ scaled version of the cross in Figure~\ref{fig:cross}). However, in this section we introduce a sufficient property of a shape $S$ to have, such that its scaled version $S_{\times 2}$ is foldable by a Turning Machine (leading to Theorem~\ref{thm:yw-separator}).

Let $C=\{c_1,c_2,\dots,c_k\}$ be a chain of grid-vertices in $S$.
That is, $C$ is a non-self-intersecting sequence of grid-points such that $c_i$ and $c_{i+1}$ are neighboring grid vertices, for all $i$.

\begin{definition}
We say that chain $C$ is a \emph{$yw$-chain} if the $y$-coordinates of $c_i$ and $c_{i+1}$ differ by exactly one.
\end{definition}
Note that in this case $C$ is a chain of straight-line segments parallel to one of the $\vec y$ or $\vec w$ directions.

Now, let $S$ be a $y$-monotone shape, let $y_\mathrm{min}$ be the $y$-coordinate of its bottom most row, and let $y_\mathrm{max}$ be the $y$-coordinate of its topmost row.

\begin{definition}\label{def:yw-separator}
We say that chain $C=\{c_1,c_2,\dots,c_k\}$ is a \emph{$yw$-separator} of a $y$-monotone shape $S$ if (1)~$C$ is a $yw$-chain, (2)~$c_i\in S$ for all $i$, and (3)~the $y$-coordinates of $c_1$ and $c_k$ are $y_\mathrm{min}$ and $y_\mathrm{max}$, respectively.
\end{definition}

First, we prove that a scaled version of an $xy$-connected $y$-monotone shape with a $yw$-separator can be partitioned into two pieces with ``nice'' left and right boundary, respectively (see Figure~\ref{fig:yw-monotone}~(right)).

\begin{definition}
Define the \emph{left (right) boundary} of a $y$-monotone shape $S$ to be the set of the leftmost (rightmost) grid-points in every row of $S$.
\end{definition}
\begin{definition}
Given the $yw$-separator $C$ of a $y$-monotone shape $S$. Define $C_{\times 2}$ as the set of grid points in $S_{\times 2}$ that corresponds to $C$ in $S$, that is, for each $c_i = (x_i,y_i) \in C$, we denote the four points in $C_{\times 2}$ corresponding to $c_i$, as $c_{i,(0,0)}=(2x_i,2y_i)$, $c_{i,(1,0)}=(2x_i+1,2y_i)$, $c_{i,(0,1)}=(2x_i,2y_i+1)$, and $c_{i,(1,1)}=(2x_i+1,2y_i+1).$
\end{definition}

Consider the $yw$-separator $C$ of $S$, where $S$ is a $xy$-connected $y$-monotone shape, and consider the grid-points $C_{\times 2}$ of $S_{\times 2}$.

\begin{figure}[t]
\centering
\includegraphics[page=2]{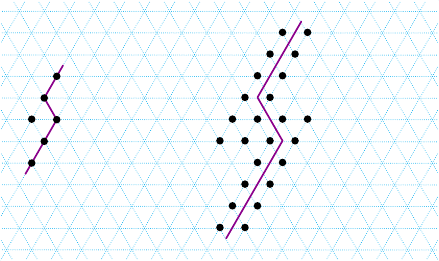}
\caption{Left: an $xy$-connected $y$-monotone shape $S$ containing 6 points, along with a $yw$-separator $C$ (in purple). 
Right: grid points in black show  $S_{\times 2}$ which is the factor-2 scaling of $S$ and the separator $C$ (in purple) that separates $S_{\times 2}$ into $S'_{\times2}$ (left) and $S''_{\times2}$ (right).}
\label{fig:yw-separator}
\end{figure}

We cut $S_{\times 2}$ in between the points of $C_{\times 2}$ in the following way (refer to Figure~\ref{fig:yw-separator}).
Let $C'$ and $C''$ be the rightmost boundary of $S'_{\times 2}$ and the leftmost boundary of $S''_{\times 2}$ respectively.
Assign the bottommost points $c_{1,(0,0)}$ to $C'$ and $c_{1,(1,0)}$ to $C''$.

For two consecutive points $c_i$ and $c_{i+1}$ on $C$, if $\overline{c_i c_{i+1}}$ is parallel to $\vec{y}$, then assign $c_{i,(0,1)}$ and $c_{i+1,(0,0)}$ to $C'$, and $c_{i,(1,1)}$ and $c_{i+1,(1,0)}$ to $C''$.

Otherwise, if $\overline{c_i c_{i+1}}$ is parallel to $\vec{w}$, then one of the following cases holds: either the point $c_i-(1,0)$ belongs to $S$, or the point $c_{i+1}+(1,0)$ belongs to $S$.
Indeed, as $S$ is a $xy$-connected $y$-monotone shape, one of the two points must belong to $S$.
Then, in the first case, when $c_i-(1,0)$ belongs to $S$, we assign one point $c_{i+1,(0,0)}$ to $C'$, and three points $c_{i,(0,1)}$, $c_{i,(1,1)}$, and $c_{i+1,(1,0)}$ to $C''$.
In the second case, when $c_{i+1}+(1,0)$ belongs to $S$, we assign three points $c_{i,(0,1)}$, $c_{i+1,(0,0)}$, and $c_{i+1,(1,0)}$ to $C'$, and one point $c_{i,(1,1)}$ to $C''$.

Finally, assign the topmost points $c_{k,(1,0)}$ to $C'$ and $c_{k,(1,1)}$ to $C''$.

In the following lemma we prove that the described method indeed partitions $S_{\times 2}$ into two simply-connected pieces with $yw$-chains on the right and left boundary.

\begin{lemma}\label{lem:cut-yw-separator}
Let $S$ be an $xy$-connected $y$-monotone shape that has a $yw$-separator.
Then $S_{\times 2}$ can be partitioned into two simply-connected $y$-monotone pieces $S'_{\times 2}$ and $S''_{\times 2}$ such that (1)~every row of $S_{\times 2}$ is partitioned into two non-empty subsets, with the left belonging to $S'_{\times 2}$, and the right to $S''_{\times 2}$, and (2)~the right boundary of $S'_{\times 2}$ and the left boundary of $S''_{\times 2}$ form $yw$-chains.
\end{lemma}
\begin{proof}
We argue that by splitting $S_{\times 2}$ between the two chains $C'$ and $C''$ constructed above, we indeed obtain such $S'_{\times 2}$ and $S''_{\times 2}$.

First, observe that $C'$ and $C''$ are both $yw$-chains.
Indeed, for each pair of rows in $S_{\times 2}$ corresponding to a row in $S$ from $y_\mathrm{min}$ to $y_\mathrm{max}$ both $C'$ and $C''$ have grid-points assigned to them (from the points corresponding to either grid-points in $C_{\times 2}$ or adjacent grid-points in direction $\vec{x}$ or $-\vec{x}$).
Furthermore, by construction, for every two consecutive points $c'_{j},c'_{j+1}\in C'$ (and $c''_{j},c''_{j+1}\in C''$) they differ in $y$-coordinate by exactly $1$, and they differ in $x$-coordinate by $0$ or $-1$.

Furthermore, observe that $C'$ and $C''$ are both $yw$-separators of $S_{\times 2}$, as they span from the bottommost to the topmost row of $S_{\times 2}$.

Let $S'_{\times 2}$ consist of $C'$ and all the grid-points of $S_{\times 2}$ to the left of $C'$, and let $S''_{\times 2}$ consist of $C''$ and all the grid-points of $S_{\times 2}$ to the right of $C''$.
Shapes $S'_{\times 2}$ and $S''_{\times 2}$ are simply-connected, as $C'$ and $C''$ are connected, and have $yw$-chains on the rightmost and the leftmost boundary respectively.
\end{proof}

We next define a traversal $\Sigma$ of $S_{\times 2}$. 
Consider the shapes $S'_{\times 2}$ and $S''_{\times 2}$ constructed as described in the proof of the above lemma.
For each pair of rows of $S_{\times 2}$ corresponding to a row in $S$, we traverse the bottom of the two rows of $S'_{\times 2}$ from the grid-point on $C'$ left to the end of the row, and return to the grid-point on $C'$ along the top one of the two rows (see Figure~\ref{fig:yw-monotone}).
As the two leftmost grid-points of the two rows in $S'_{\times 2}$ correspond to the same grid-point of $S$, the traversal from the first to the second row is valid.
Similarly, we traverse the bottom of the two rows of $S''_{\times 2}$ from the grid-point on $C''$ right to the end of the row, and return to the grid-point on $C''$ along the top one of the two rows.
As $C'$ and $C''$ are $yw$-chains, the traversal between the two pairs of rows in $S'_{\times 2}$ and $S''_{\times 2}$ is valid as well.
Finally, we connect the two parts of the traversal $\Sigma$ across $C'$ and $C''$ in the bottommost row of $S_{\times 2}$.

\begin{theorem}\label{thm:yw-separator}
Let $S$ be an $xy$-connected $y$-monotone shape that has a $yw$-separator $C$, then $S_{\times 2}$ is foldable by a Turning Machine in expected time $O(\log n)$.
\end{theorem}
\begin{proof}
We prove that the Turning Machine $T^\Sigma$ defined as follows folds the shape $S_{\times 2}$, i.e.~on all trajectories.
Let $T^{\Sigma}$ consist of $n$ monomers $m_0,m_1,\dots,m_{n-1}$, where $n$ is the number of grid-points of $S_{\times 2}$, and the first monomer $m_0$ corresponds to the topmost grid-point of $C'$ (defined above). 
In the folded state, $T^{\Sigma}$ traverses $S'_{\times 2}$ from top to bottom, and then traverses $S''_{\times 2}$ from bottom to top.
Let the initial states of the monomers forming the bottommost row (except for the rightmost one) of $S_{\times 2}$ be~$0$.
Applying Lemma~\ref{lem:state-chain-general} we derive the initial states of the remaining monomers (refer to Figure~\ref{fig:yw-monotone}).
Let monomer $m_j$ correspond to a grid-point $p_j\in S'_{\times 2}$.
We consider the following cases: $p_j$ is in an odd or even row of $S'_{\times 2}$; furthermore, we consider the cases when $p_j$ is the leftmost point, the rightmost point, or an interior point of the row.
Let $C'=\{c'_1,c'_2,\dots,c'_k\}$ and $C''=\{c''_1,c''_2,\dots,c''_k\}$, where $c'_1$ and $c''_1$ are the bottom most grid-points, and $c'_k$ and $c''_k$ are the top most grid-points.
Let $p_j$ be in an odd row, then
\begin{itemize}
\item if $p_j=c'_\ell\in C'$ and $\ell>1$, then $s(m_j)=-2$ if $\overline{c'_{\ell-1},c_\ell}$ is parallel to $\vec{y}$, and $s(m_j)=-1$ if $\overline{c'_{\ell-1},c_\ell}$ is parallel to $\vec{w}$,
\item otherwise, if $p_j\not\in C'$ or $p_j=c'_1$, then $s(m_j)=0$.
\end{itemize}
Let $p_j$ be in an even row of $S'_{\times 2}$, then
\begin{itemize}
\item if $p_j$ is the leftmost point in its row, then $s(m_j)=-2$,
\item otherwise, $s(m_j)=-3$.
\end{itemize}
Similarly, for the monomers whose final positions fall in $S''_{\times 2}$, we specify the following initial states.
Let $p_j$ be in an even row of $S''_{\times 2}$, then
\begin{itemize}
\item $s(m_{n-1})=0$,
\item if $p_j=c''_\ell\in C''$ and $\ell<k$, then $s(m_j)=1$ if $\overline{c''_{\ell},c''_{\ell+1}}$ is parallel to $\vec{y}$, and $s(m_j)=2$ if $\overline{c'_{\ell},c''_{\ell+1}}$ is parallel to $\vec{w}$,
\item otherwise, if $p_j\not\in C''$, then $s(m_j)=-3$.
\end{itemize}
Let $p_j$ be in an odd row of $S''_{\times 2}$, then
\begin{itemize}
\item if $p_j$ is the right point in its row, then $s(m_j)=1$,
\item otherwise, $s(m_j)=0$.
\end{itemize}
We claim that $T^\Sigma$ indeed computes $S_{\times 2}$ and does not enter a permanently blocked configuration for any sequence of transition rules (trajectory).
Assume that there is a configuration $c$ that is reachable from the initial configuration of $T^\Sigma$ that is permanently blocked.
Let $m_i$ be the monomer corresponding to the bottommost grid-point of $C'$, that is, $m_i$ is the last monomer in the traversal $\Sigma$ that is still in $S'_{\times 2}$.
Let $T'$ and $T''$ be Turning Machines consisting of the monomers $\{m_0,\dots,m_i\}$ and $\{m_{i+1},\dots,m_{n-1}\}$ correspondingly.
Similarly to the proofs of Theorems~\ref{thm:zz} and~\ref{thm:y-monotone-scale1}, we can argue that $T'$ and $T''$ individually compute the shapes $S'_{\times 2}$ and $S''_{\times 2}$, i.e., they fold without entering in a permanently blocked configuration.
Thus, if $c'$ is blocked, then for every monomer $m_\ell$ there must be a pair of monomers $m_j$ with $j\le i$ and $m_k$ with $k>i$ that are blocking the transition of $m_\ell$.

\begin{figure}[t]
\centering
\includegraphics[page=3]{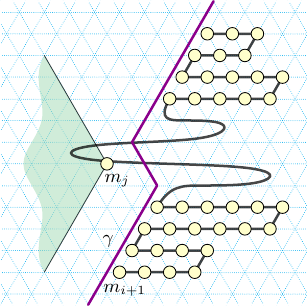}
\caption{Folding of a Turning Machine $T''$. Its monomers remain on the right side of the cut $\gamma$. If monomer $m_j$ were to the left of $\gamma$, its further position is confined to the cone shown in green.}
\label{fig:yw-zigzag}
\end{figure}

Consider individual folding of the Turning Machine $T''$.
Without loss of generality, let the monomer $m_{i+1}$ be fixed in the origin of the coordinate system, and the remaining monomers move with respect to it.
We consider the cut $\gamma$ between the shapes $S'_{\times 2}$ and $S''_{\times 2}$ in this global coordinate system, and extend its ends to infinity along the $\vec{y}$ direction (see Figure~\ref{fig:yw-zigzag}).
We claim that the monomers of $T''$ always remain to the right of $\gamma$.
Consider how the position $\pos(m_j)$ changes when the transition rules are applied to the monomers of $T''$.
The monomers of $T''$ rotate from orientation $\vec{x}$ to $\vec{y}$, from $\vec{y}$ to $\vec{w}$, and from $\vec{w}$ to $-\vec{x}$.
If a monomer $m_k$ moves, where $k\ge j$, the position of $m_j$ does not change.
If a monomer $m_k$ moves, where $k<j$, the position of $m_j$ changes by one unit distance in the directions $\vec{w}$, $-\vec{x}$, or $-\vec{y}$.
Suppose in an intermediate configuration $c''$ some monomer $m_j\in T''$ has a position to the left of the cut $\gamma$.
From this moment on, the position of $m_j$ is confined to a $120^\circ$ with its apex in $\pos(m_j)$ and the two rays emanating in the directions $\vec{w}$ and $\vec{-y}$.
Thus, monomer $m_j$ will never cross the cut $\gamma$ to the right side, and thus $T''$ does not compute $S''_{\times 2}$.

Similarly, we argue that the monomers of $T'$ always remain to the left of the cut $\gamma$.
This implies that there can never be a blocked pair $m_j$ with $j\le i$ and $m_k$ with $k>i$.
Therefore, $T^\Sigma$ folds $S_{\times 2}$ on every trajectory.

For the time analysis, we note that the target configuration of $T'$ is a negative zig-zag path, and the target configuration of $T''$ is a positive zig-zag path. By Theorem~\ref{thm:zz}, both $T'$ and $T''$ would each individually fold in $O(\log n)$ expected time if they were separate Turning Machines. In the construction in this proof, we show that monomers in $T'$ don't block any monomers in $T''$ (and vice-versa), hence even if both $T'$ and $T''$ are joined, they fold independently. 
Thus, $T^\Sigma$ completes in $O(\log n)$ expected time.
 \end{proof}

\subsection{Shapes that have a traversal, yet are not foldable: spirals}\label{sec:impossible-shapes} 
\begin{figure}[t]
    \centering
    \includegraphics[width=0.65\textwidth]{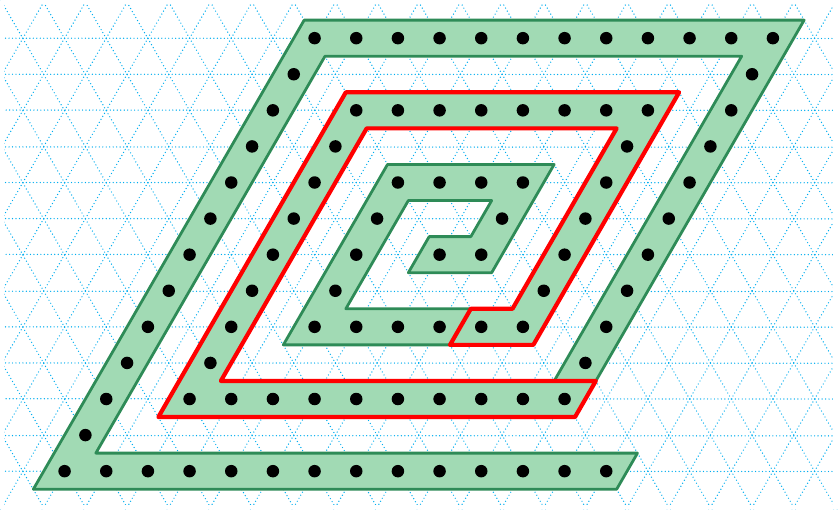}
    \caption{A $k=3$-turn, gap 1, spiral on the triangular grid $G_\triangle$, as defined in Definition~\ref{def:spiral}. The centre of the spiral is at (0,0), and the arms of the spiral turn anti-clockwise while keeping at a distance (gap) of exactly 1 from each other.  The monomers framed in red are those which reside in the ``almost rectangle'' $R_{k'} = R_{2}$ in Definition~\ref{def:spiral}.}
    \label{fig:spiral_defn}
\end{figure}

The goal of this section is to prove Theorem~\ref{thm:spiral}, which gives a class of shapes that have a traversal yet are not foldable. Specifically,  the  following definition of $k$-turn, 1-gap, spiral, defines a particular kind of spiral that makes  $k$ anti-clockwise  turns, and has its arms held at unit distance (gap 1) apart, as shown in Figure~\ref{fig:spiral_defn}. 

\begin{definition}[$k$-turn $1$-gap spiral]\label{def:spiral}
For $k\in \Nset^+$, define the anti-clockwise \emph{$k$-turn $1$-gap spiral} to be $\bigcup_{k' \leq k} R_{k'}$, where $R_{k'}$ is an ``{\em almost rectangle}'' formally defined as follows (see Figure~\ref{fig:spiral_defn} for a pictorial definition): 
\begin{align*}
R_{k'} &=  \{ (x,y) \mid y = \pm 2k' \textrm{ and } -2k' \leq x \leq 2k'-1 \}  \\
          &\cup  \{ (x,y) \mid x \in \{ -2k' , 2k'-1\} \textrm{ and } -2k' \leq y \leq 2k' \}  \\
          & \cup \{ (2k'-2, -2k' + 2), (2k',-2k'),(2k'+1,-2k') \} \\
          &\setminus \{ (2k'-1, -2k'  +1  ) \}\,.
\end{align*}
\end{definition}

\begin{remark}
For each $k$, the $k$-turn $1$-gap spiral has two traversals--one starting at the centre, one ending at the centre. 
It can also be seen that $k$-turn $1$-gap spiral has a Turning Machine that has a trajectory whose final configuration is the spiral---$S_M$ starts as a line with monomer $m_0$ at the centre of the spiral, and with initial states being {\em turning numbers}\footnote{Start with some integer (e.g. 0) and traverse the curve writing down an integer at each point, but incrementing the noted value at each turn of $\pi/3$ anti-clockwise, and decrementing at each turn of $\pi/3$ clockwise.} for the spiral curve. 
One can then imagine a trajectory that folds each arm of the spiral one at a time starting at the centre and working its way out (like rolling up a piece of paper).  
However, despite this, in this section we show that for each $k \in \Nset^+$, the $k$-turn $1$-gap spiral is not foldable by a Turning Machine (Theorem~\ref{thm:spiral}). 
The proof shows that any Turning Machine that attempts to fold a spiral must fail by either having invalid/illegal states, or else having at least one permanently blocked reachable configuration. 
\end{remark}

The following technical definition is used in the statement of Lemma~\ref{lem:turning numbers for a spiral}. 
The notation $[a]^b$, means $\underbrace{a, a,\ldots , a}_{b \textrm{ times}}$. 

\newcommand{\Tio}{\ensuremath{T_{\mathrm{in-to-out}}}}
\newcommand{\Toi}{\ensuremath{T_{\mathrm{out-to-in}}}}
\begin{definition}\label{def:turning numbers for a spiral}
For $k\in\Nset^+$, let $S$ be a $k$-turn $1$-gap spiral.
A sequence of inside-to-outside turning numbers for S is of the form
\[
\Tio(t_0) 
= [t_0]^{1}, [t_1]^{2}, [t_2]^{3}, [t_3]^{4}, \ldots, [t_{4k-2}]^{4k-1}, [t_{4k-1}]^{4k}, [t_{4k}]^{4k+1}, t_{4k}  \,,
\]
where $t_0 \in\Zset$ is such that $t_0 \equiv 0~\mathrm{mod}~6$ and 
\begin{equation}\label{eq:Tio}
t_i = \begin{cases} 
t_{i-1} + 2 & \textrm{ if } i \textrm{ is even, }  \\
t_{i-1} + 1 & \textrm{ if } i \textrm{ is odd. }
\end{cases} 
\end{equation}
A sequence of outside-to-inside turning numbers for $S$ is any sequence of the form
\[
\Toi(t_0) = 
t_0 , [t_0]^{4k+1}, [t_1]^{4k}, [t_2]^{4k-1}, [t_3]^{4k-2}, \ldots  ,  [t_{4k-1}]^{2} ,  [t_{4k}]^{1}  \,,
\]
where $t_0 \in\Zset$ is such that $t_0 \equiv 3~\mathrm{mod}~6$, and 
\begin{equation}\label{eq:Toi}
t_i = \begin{cases} 
t_{i-1} - 1 & \textrm{ if } i \textrm{ is even, }  \\
t_{i-1} - 2 & \textrm{ if } i \textrm{ is odd. }
\end{cases} 
\end{equation}
\end{definition}

\begin{lemma}\label{lem:turning numbers for a spiral}
For $k\in\Nset^+$, the $k$-turn $1$-gap spiral $S_k$ is unfoldable by any Turning Machine $M$ that has initial state sequence  $s_0(m_0),s_0(m_1),\ldots, s_0(m_{n-1}) $ 
that is not one of the turning number sequences $\Tio(s_0(m_0))$ or $\Toi(s_0(m_0))$ from Definition~\ref{def:turning numbers for a spiral}.
\end{lemma}
\begin{proof}
From its definition, $S_k$ has exactly $|S_k| = 8k^2+6k+2$ points, which, by a straightforward calculation, is the same as the length of  
the sequences of canonical turning numbers in Definition~\ref{def:turning numbers for a spiral}. 
If $M$ does not have exactly $n = |S_k|$ monomers, does not fold $S_k$ (either it does not cover all points of $S_k$ or covers too many).  
Assume then that we are given a Turning Machine $M$ with $n = |S_k|$ monomers, but that do not have states given by Definition~\ref{def:turning numbers for a spiral}.

The spiral has exactly two points, namely $(0,0)$ and $(2k+1, -2k)$, that have degree 1 (one neighbour in $S_k$); we respectively call them the {\em inside} and {\em outside} start points. 
If the spiral were foldable, then monomer $m_0$ is positioned on either the inside or outside start point. 
By Lemma~\ref{lem:s2}, $| s_0(m_i) - s_0(m_{i+1})| \leq  2$ for all $0\le i<n-1$.

Suppose $\pos(m_0)$ is the inside start point. 
If any of the remaining claims do not hold, then the $S_k$ is unfoldable: 

Monomer $m_0$'s initial state is $s_0(m_0) \equiv 0 \mod 6$, by directionality of $m_0$ in the final configuration of $M$ (if not, we are done because either $m_0$ finishes in state 0 but pointing in the wrong direction and thus places $\pos(m_1)$ outside of $S_k$, or else $m_0$ never reaches state 0 meaning $m_0$ is permanently blocked).
Also,  $s_0(m_1) = 1 \mod 6$, for the same reason. 
But Lemma~\ref{lem:s2} tell us that $| s_0(m_i) - s_0(m_{i+1})| \leq  2$ for all $0\le i<n-1$, 
hence $s_0(m_1) = s_0(m_0) + 1$. (If not, we would get that $M$ is blocked (by violating Lemma~\ref{lem:s2}), or that $M$ does not precisely trace the spiral (if $s_0(m_1) \neq 1 \mod 6$). 
Tracing around $S_k$ using the same reasoning for each point along $S_k$
gives that either $\Tio(s_0) = s_0(m_0),s_0(m_1),\ldots, s_0(m_{n-1}) $ or else $M$ does not fold $S_k$, giving the lemma conclusion for this case. 

Else, $\pos(m_0)$ is on the outside start point of $S_k$. 
A similar argument (to the inside case) shows that  
$\Toi(s_0) = s_0(m_0),s_0(m_1),\ldots, s_0(m_{n-1}) $ or else $M$ does not fold $S_k$, giving the lemma conclusion for this case.
\end{proof}

\begin{theorem}[shapes with a traversal, but are unfoldable]\label{thm:spiral}
For all $k\geq 2$ the $k$-turn $1$-gap spiral $S_k$ is not foldable. 
\end{theorem}

\begin{figure}[t]
  \centering
\includegraphics[page=1,width=\textwidth]{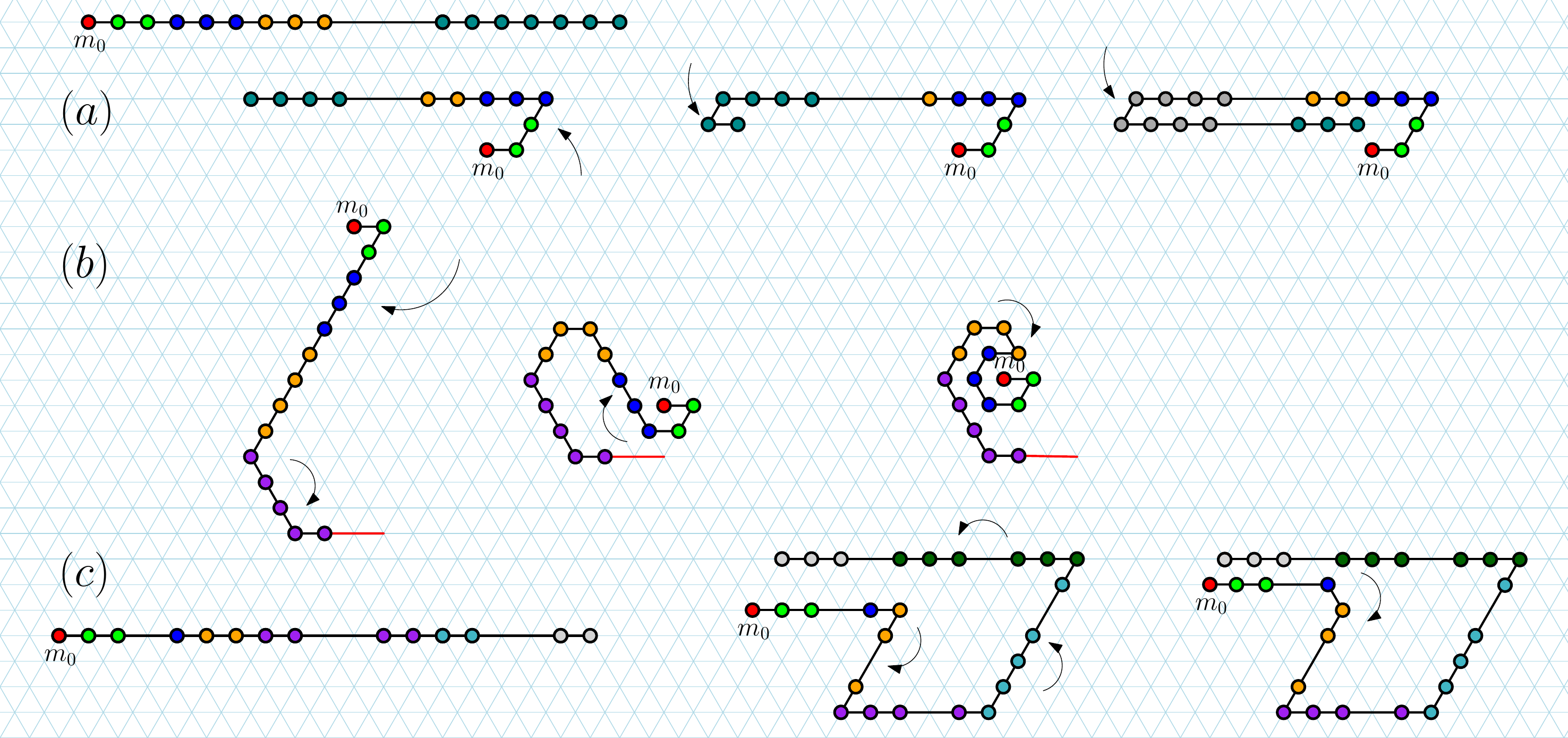}
\caption{Illustration for Case 1 of the proof of Theorem~\ref{thm:spiral}, showing that a $2$-turn, 1-gap, Spiral can not be folded from ``inside to out''. 
Top left: initial configuration, monomers are co-linear, colours distinguish initial states.
(a)~Initial configuration with all states positive.  First we fold the intended innermost coil/turn of the spiral, by rotating almost the full line of monomers anti-clockwise by $\pi$, creating a small ``pocket'' (red, green, and blue monomers). 
Then the teal/grey monomers ``prematurely'' fold the configuration into the pocket, eventually yielding a permanently blocked configuration -- to carry out the folding it suffices that the teal monomers have initial state $\geq 6$. 
(b)~Initial configuration with all states negative. The trajectory shown folds into a permanently blocked configuration similar to that in Figure~\ref{fig:T7}, but here uses clockwise turning (negative states) and has an extra long tail that needs to be folded out of the way before completing the erroneous tight spiral. The configuration in on the right is not (yet) permanently blocked, but the 7 monomers in the inner tight spiral are permanently blocked and thus all trajectories forward lead to permanently blocked configurations.  
(c)~Initial configuration with strictly positive and strictly negative states. By Case 1(c) in the proof, there is a contiguous segment of monomers in state 0, here shown in purple. Monomers to the left (negative states) fold clockwise, and monomers to the right (positive states) fold anticlockwise, yielding a pincer-like permanently blocked configuration.} 
    \label{fig:Tio}
\end{figure}

\begin{proof}
Suppose for the sake of contradiction that there is a Turning Machine $M$ that folds $S_k$, $k\geq 2$.
By Lemma~\ref{lem:turning numbers for a spiral}, there are two cases, either,
(Case 1)  
$M$ has initial states $\Tio(s_0(m_0))$ and $s_0(m_0) \equiv 0 \mod 6$, or 
 (Case 2) $\Toi(s_0(m_0))$ and $s_0(m_0) \equiv 3 \mod 6$.

\vspace{\baselineskip} \noindent {\bf Case 1}: $M$ has initial states $\Tio(s_0(m_0))$ for $s_0(m_0) \equiv 0 \mod 6$ (intuitively, folding the spiral from inside to outside).

We have $\pos(m_0)=(0,0)$ (at the ``centre'' of the spiral), 
otherwise if $M$ finishes with all monomers in state 0 it would place (many) monomers outside $S_k$ and we get the statement.
There are three subcases: 
(a) all initial states are positive, (b) all initial states are negative, or (c) there are both strictly positive and strictly negative initial states:

\begin{itemize}
\item Case 1(a): all initial states are positive ($s_0(m_i) \geq 0$ for all $0\leq i < n$).\\
Here, the idea is find a reachable configuration $c_b$ of $M$ that is permanently blocked, as illustrated in Figure~\ref{fig:Tio}(a).  

First, the $M$ carries out repeated line rotations by $\pi / 3$, until reaching a configuration where $m_0$ is in state 0.  
Then the line of monomers $m_1,m_2,m_3,\ldots$ rotates by $\pi / 3$, which puts $m_1,m_2$ in state 0.
We next have the line of monomers $m_3,m_4,m_5,\ldots$ rotate by $2 \pi / 3$, which puts $m_1,m_2$ in state 0, and parallel to the x-axis, pointing the $-x$ direction  (Figure~\ref{fig:Tio}(a), left).

Next we apply a turning rule application to $m_{n-2},m_{n-1}$, and then another to $m_{n-1}$ (Figure~\ref{fig:Tio}(a), middle).
Now we repeat the previous three steps, moving backwards along monomer indices, until we reach the configuration shown in Figure~\ref{fig:Tio}(a), right:  
i.e., for each $j \leq n-1$ (starting at $j=n-1$) apply one turning rule to $m_{j-2}$, one to $m_{j-1}$ and then one more turning rule to $m_{j-2}$, then decrement $j$ and repeat. 
Eventually we reach a value of $j$ for which no move is possible, yielding the permanently blocked configuration shown in Figure~\ref{fig:Tio}(a), right.

\item Case 1(b): all initial states are negative  ($s_0(m_i) \leq 0$ for all $0\leq i < n$).\\  
Here, the intuition is to block $M$ by using the fact that $2\pi$ line rotation is impossible ($k\geq 2$), as shown in Figure~\ref{fig:Tio}(b).

First, by Definition~\ref{def:turning numbers for a spiral},  $t_0<t_1<\cdots<t_{4k}$ on  $\Tio(t_0)$. 
Moreover, for all monomer states to be negative, and by applying Equation~\eqref{eq:Tio}, it is case that $t_0\leq -6k$. 
By the definition of \Tio($t_0$) and since $k=2$, there are $\geq 15$ monomers with initial state $\leq -6$. 
Next we apply the main idea used to prove Theorem~\ref{thm:no 2pi rotation}, and illustrated in Figure~\ref{fig:T7} (there we folded in the anticlockwise direction, here it is clockwise). 
Specifically, we  claim that the configuration shown in Figure~\ref{fig:Tio}(b), right, is reachable, 
since the 15 monomers $m_0, m_1,\ldots m_14$ have initial state $\leq -6$, hence they are able to reach the configuration shown,
and since the other monomers (orange, purple) can be moved as shown to enable the tight inner spiral to form.  
The configuration in Figure~\ref{fig:Tio}(b), right is  not (yet) permanently blocked, but the 7 monomers in the inner tight spiral are permanently blocked and thus all trajectories forward lead to permanently blocked configurations. 

\item Case 1(c):  there are both strictly positive and strictly negative initial states.\\
We know from the definition of $\Tio(t_0)$, and the subcase we are in, that $t_0 < 0 <  t_{4_k} $. 
Moreover, we claim that $\Tio(t_0)$ has a contiguous subsequence $0,0,\ldots,0$; 
in other words there is an $i \in \{1,2,\ldots,4k-1\}$ such that the initial state sequence, $\Tio(s_0) = \Tio(t_0) = t_0,\ldots t_{4_k}$,
contains the state subsequence $ [t_i]^{i+1} = [0]^{i+1}$. 
To see the claim, first note that by Definition~\ref{def:turning numbers for a spiral},  
$t_0 \equiv 0 \mod 6$, and by Equation~\eqref{eq:Tio} every 4th term   
$t_{4j}$ (here, $0< j < k$) has the property that  $t_{4j} \equiv 0 \mod 6$. 
Second, since $t_0 < 0$ and $0 < t_{4k}$, there is some $j$ such that $t_{4j} = 0$, giving the claim.

Let $l_i$ denote the segment of $i+1$ monomers with initial states $0,0,\ldots,0$ (Figure~\ref{fig:Tio}(c) purple segment).
Then, the preceding segment $l_{i-1}$ has length $|l_i|-1 = i$ and monomers all in initial state $-2$, 
and all monomers preceding that have initial state $<-2$  (by Equation~\eqref{eq:Tio}). 
Hence the second configuration in Figure~\ref{fig:Tio}(c) is reachable (decrement each monomer in $l_{i-1}$ twice). 
 The succeeding segment $l_{i+1}$ has length $|l_i|+1 + i+1$ and monomers all in state $1$, and  
and all monomers succeeding that have initial state $\geq 3$ (by Equation~\eqref{eq:Tio}). 
Hence the third configuration in Figure~\ref{fig:Tio}(c) is reachable (increment each monomer in $l_{i+1}$ once, and each monomer
in $l_j$, for $j>i+1$, three times). 
Figure~\ref{fig:Tio}(c) is permanently blocked.  
\end{itemize}

\begin{figure}[t]
  \centering
\includegraphics[width=\textwidth]{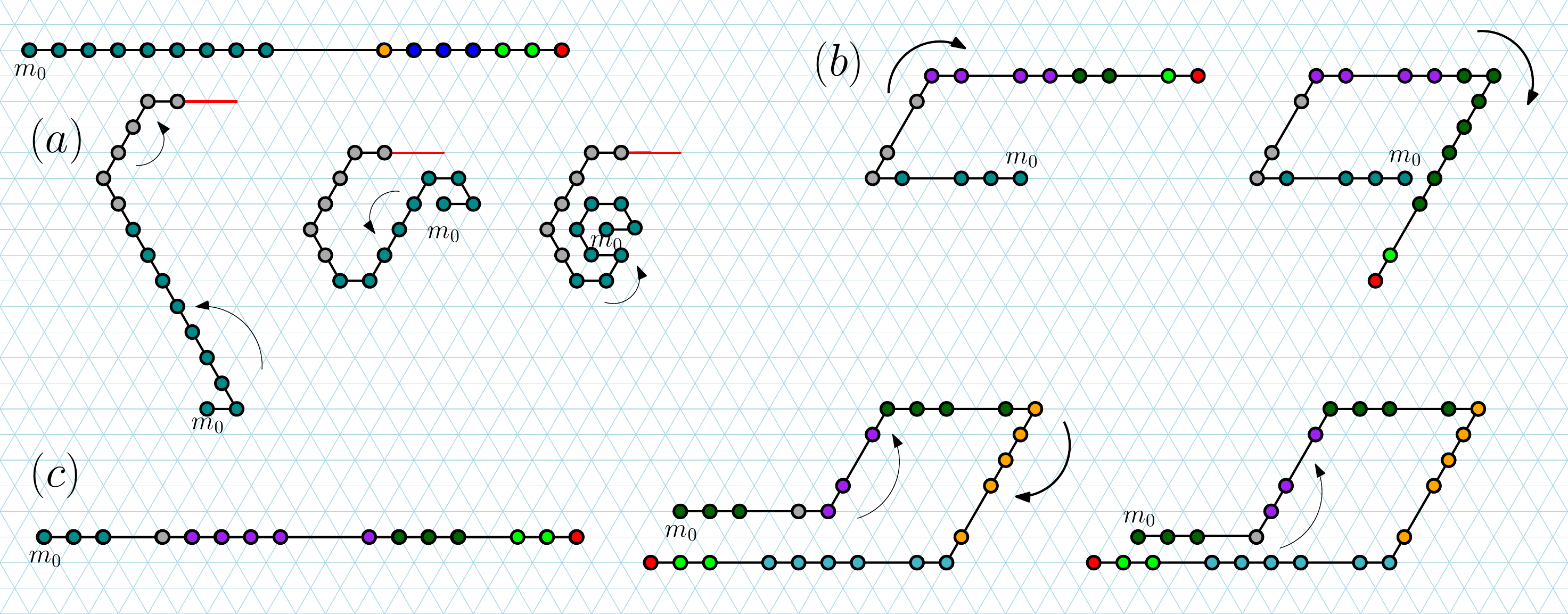}
\caption{Illustration for Case 2 of the proof of Theorem~\ref{thm:spiral} showing that a 2-turn, 1-gap, Spiral can not be folded from ``outside to in". Top left: initial configuration, monomers are co-linear, colours distinguish initial states. (a) Initial configuration with all states positive. Trajectory shown folds into a permanently blocked configuration similar to that in Figures~\ref{fig:Tio}(b) and~\ref{fig:T7}. The configuration to the right is not yet permanently blocked, but the 7 monomers in the inner tight spiral (starting at $m_0$) are permanently blocked and thus all trajectories from this configuration will yield permanently blocked configurations. 
(b) Initial configuration with all states negative. First we fold the outer most part of the spiral by rotating almost a full line of monomers clockwise by $\pi$ creating a large ``C shape'' (dark green, gray, purple monomers). Then the remaining arm of the spiral folds such that it is anti-parallel to the $y$-axis (right). This arm of the spiral wishes to rotate clockwise to reach the center, but to do so would cause these monomers to intersect with the ``C shape'' outer coil of the spiral, thus leading to a permanently blocked configuration. 
(c) Initial configuration with both strictly positive and strictly negative initial states. This construction is similar to that used in Figure~\ref{fig:Tio}(c) such that there is a contiguous segment of monomers with states $0,0,\ldots,0$, here shown in dark green. Monomers to its left (positive states) fold anti-clockwise, and monomers to the right (negative states) fold clockwise, yielding a ``pincer-like'' permanently blocked configuration.} 
    \label{fig:Toi}
\end{figure}

\bigskip
\noindent{\bf Case 2}: $M$ has initial states $\Toi(s_0(m_0))$ for $s_0(m_0) \equiv 3~\mathrm{mod}~6$ (intuitively, folding the spiral from outside to inside).

We have $\pos(m_0)= (2k+1, -2k)$ (at the ``outside'' of the spiral, bottom right corner), 
otherwise if $M$ finishes with all monomers in state 0 it would place (many) monomers outside $S_k$ and we get the statement.
There are three subcases: 
(a) all initial states are positive, (b) all initial states are negative, or (c) there are both strictly positive and strictly negative initial states: 

\begin{itemize}
\item Case 2(a): all initial states are positive ($s_0(m_i) \geq 0$ for all $0\leq i < n$).\\ 
this case is illustrated in Figure~\ref{fig:Toi}(a). 
The proof here is very similar to that of Case 1(b), in particular Figure~\ref{fig:Toi}(a) is a reflection about the x-axis of Figure~\ref{fig:Tio}(b),
 and the only other difference here is that to reach a permanently  blocked configuration  we start turning from the monomers that are intended to become the {\em shorter} (rather than longer) arms of the spiral.  

\item  Case 2(b): all initial states are negative  ($s_0(m_i) \leq 0$ for all $0\leq i < n$).\\ 
First, by Definition~\ref{def:turning numbers for a spiral},  $t_0 > t_1 > \cdots > t_{4k}$ on  $\Toi(t_0)$. 
Moreover, for all monomer states to be negative, and by applying Equation~\eqref{eq:Toi}, it is case that $t_0\leq -3$. 
We begin by applying $|t_0|$ turning rule applications to {\em all} monomers, until the first $k$ monomers (i.e, $m_0,m_1,\dots,m_{4k}$) have each reached state 0. 
Next we apply 2 turning rules to all non-zero monomers such that monomers with initial state $t_1$ turn to state zero (grey in Figure~\ref{fig:Toi}(b), left). 
Next we apply 1 turning rules to the remaining non zero monomers such that the monomers which were initially in state $t_2$ are now in state zero (purple in Figure~\ref{fig:Toi}(b), left). 
Next we apply 2 turning rule applications to monomers $m_i$ with $12k+2<i<n$ in order to reach the `q-shaped' configuration shown in Figure~\ref{fig:Toi}(b), right. 

We claim that all trajectories forward from this configuration reach a permanently blocked configuration:  
The first three ``spiral arms'' (monomers with initial states $t_0,t_1,t_2$) have reached state 0 and thus will not rotate any further. 
In order for the remaining non-0 state monomers to attempt reach the target configuration of $S_k$, they can only make clockwise turns which results in them wrapping around  the blocked inner region. Eventually all are permanently blocked 

\item Case 2(c):  there are both strictly positive and strictly negative initial states. This case is similar to that of Case 1(c). By the definition of \Toi($t_0$), and the subcase we are in, $t_0>0>t_{4k}$. Similarly to Case 1(c) we claim that the initial state sequence \Toi($s_0$) = \Toi($t_0$) contains the state sequence $[t_i]^{4k+1-i}=0$ for some $i \in \{1,2, \ldots ,4k-1\}$. 
To see this we first note that by Definition~\ref{def:turning numbers for a spiral}, $t_0\equiv 3 \mod 6$, and by Equation~\eqref{eq:Toi} every 4th term after the $2^\textrm{th}$ term is congruent to $0 \mod 6$ (that is to say $t_i\equiv 0 \mod 6$ such that $i \equiv 2  \mod 4$).
Thus $t_{4j+2}\equiv 0 \mod 6$ for some j with $0\leq j < k$. We let $l_i$ denote this segment of $4k+1-i$ monomers each with initial state 0. (Figure~\ref{fig:Toi}(c) dark green segment). Then the preceding segment $l_{i-1}$ has length $4k-i+2$ and has initial states $1$ and the succeeding segment $l_{i+1}$ of length $4k+2$ with initial state $-2$ by  Equation~\eqref{eq:Toi}. Hence the second configuration in Figure~\ref{fig:Toi} is reachable. The segment $l_{i+2}$ of length $4k-i-1$ has initial state $-3$ by Equation~\eqref{eq:Toi}. Hence the third configuration in Figure~\ref{fig:Toi}(c) is reachable. Thus Figure~\ref{fig:Toi}(c) is permanently blocked.
\end{itemize}
\end{proof}

\section*{Acknowledgments}
We thank Vera Sacristán and Suneeta Ramaswami for insightful ideas and important input. This work began at the 29th Bellairs Winter Workshop on Computational Geometry (March 21--28, 2014 in Holetown, Barbados), we thank Erik Demaine for organising a wonderful workshop and providing valuable feedback, and the rest of the participants for providing a stimulating environment. We also thank Dave Doty and Nicolas Schabanel for helpful comments.
This paper expands upon the conference version published in {\em Proceedings of the 26th International Conference on DNA Computing and Molecular Programming}. Oxford 2020, Leibniz International Proceedings in Informatics (LIPIcs), with new material appearing in Sections~\ref{sec:y-monotone}--\ref{sec:impossible-shapes}. 
We thank the anonymous reviewers of this work for insightful comments led to significant improvements of the text.


\appendix
\input{appendix_s1}

\end{document}

%% file: appendix_s1.tex

\newpage
\section{Line rotation by \texorpdfstring{$\pi/3$}{π/3}, \texorpdfstring{$\pi$}{π} and \texorpdfstring{$4\pi/3$}{4π/3}}\label{app:sec:rot}

In this appendix we present proofs that line-rotating Turning Machine for respective angles of $\pi/3$, $\pi$ and $4\pi/3$ terminates in expected time $O(\log n)$. 
These claims are superseded by the results in the main paper, but we include the proofs as they give a number of techniques to analyse the Turning Machine model. 

\subsection{Line rotation by \texorpdfstring{$\pi/3$}{π/3}: \texorpdfstring{$\tml{n}{1}$}{L1n}}
The following proof of line rotation by $\pi/3$ radians is intended to be a simple example worked out in detail.
Let $L^1_n$ be the Turning Machine defined in Definition~\ref{def:line} with $\sigma=1$, as illustrated in Figure~\ref{fig:60} (left).

\lempioverthree*
\begin{proof}
The initial configuration (Figure~\ref{fig:60}, left) of $\tml{n}{1}$ is a line of $n-1$ monomers in state $1$ with an additional final monomer in state 0, i.e.\ at time 0 the $n$ states are $s(m_0) s(m_1) \cdots s(m_{n-1}) = 1^{n-1}0$.
Since monomer states only change by decrementing from 1 to 0, any configuration on any trajectory of $\tml{n}{1}$ has its (composite) state of the form $\{0,1\}^{n-1} 0$. 
Consider a configuration $c$ in a trajectory of evolution of $\tml{n}{1}$, and the corresponding state\footnote{In fact any $x \in \{0,1\}^{n-1} 0$ is the state of a reachable configuration, but we don't need to prove that.} $x \in \{0,1\}^{n-1} 0$.
Let $m_i^c$ denote the $i$th monomer of $\tml{n}{1}$ in configuration $c$.
For any $i \in \{0,1,\ldots,n-2\}$ such that $s(m_i^c)=1$, consider the unique configuration $c'$ where $c\rightarrow c'$ and $s(m_i^{c'})=0$ (and, by definition of next configuration step, $j\neq i$ implies $s(m_j^{c'})=s(m_j^{c})$).

We claim that
$\tail(m_i^c)$ does not share any positions with $\headp(m_i^c)$, in other words, that~$c'$ is a non-self-intersecting configuration. 
To show this, consider a horizontal line $\ell_i$ through monomer $m_i^c$ and observe that in $c'$ (and in $c$), the monomers $\tail(m_i^c) = m_0^c, m_1^c, \ldots, m_{i}^c$ lie on or below $\ell_i$ (because the path  $\pos(m_0^c), \pos(m_1^c),\ldots,\pos(m_i^c)$ is connected and consists of unit length segments each at an angle of either $0^\circ$ or $60^\circ$ clockwise relative to the $x$-axis), but the monomers $\headp(m_i^c) = m^{c'}_{i+1}, m^{c'}_{i+2}, \ldots, m^{c'}_{n-1}$ lie strictly above $\ell_i$ (because $\pos(m^{c'}_{i+1})$ is strictly higher than $\pos(m^c_i)$, and because the path  $\pos(m'_{i+1}),\pos(m'_{i+2}), \ldots,\pos(m^{c'}_{n-1})$ is connected and consists of unit length segments each at an angle of $0^\circ$ or $60^\circ$ to the $x$-axis).
Hence there are no blocked configurations reachable by $\tml{n}{1}$ (neither permanent nor temporary blocking).

At each reachable configuration $c$, starting from the initial configuration, we can choose $i$ independently from the set of non-zero states. 
The expected time for the first rule application  is $1/(n-1)$ since it is the expected time of the minimum of $n-1$ 
independent exponential random variables each with rate 1.  
The next is $1/(n-2)$, and so on.  
By linearity of expectation, the expected  value of the total time $T$ is 
\[
E[T] 
      = \sum_{k=1}^{n-1} \frac{1}{k} 
        = O(\log n)\,
\]
where the sum is the $(n-1)^{\mathrm{th}}$ partial sum of the harmonic series, known to have a $O(\log n)$ bound. 
Hence $\tml{n}{1}$ completes in expected $O(\log n)$ time.
\end{proof}

\subsection{Line rotation by \texorpdfstring{$\pi$}{π}: \texorpdfstring{$\tml{n}{3}$}{L3n}}
Next, we analyse line rotation of $\pi$ radians. 
\begin{lemma}\label{lem:pi:blocking} 
Let $\tml{n}{3}$ be a line-rotating Turning machine, then:
\begin{enumerate}[(i)]
\item\label{concl:only state 1} only monomers in state 1 adjacent to a monomer in state 3 are blocked, and that blocking is temporary,
\item\label{concl: no more than 2n/3 blocked} any reachable configuration of $\tml{n}{3}$ has no more than $2n/3$ blocked monomers, and 
\item\label{concl:pi:exactly 2n/3 blocked} there exists a configuration of $\tml{n}{3}$ that has exactly $2n/3$ blocked monomers.
\end{enumerate}
\end{lemma}
\begin{proof} 
Consider any reachable configuration $c$ of $\tml{n}{3}$, and let monomer $m_i$ be blocked in $c$.
By Lemma~\ref{lem:unblocked-states}, monomers in state 2 and 3 are never blocked.
By definition, monomers in state 0 are not blocked.
Thus if $m_i$ is blocked it is in state 1, i.e. $s(m_i)=1$. 
We claim that in this case either $s(m_{i-1})=3$ or $s(m_{i+1})=3$ (or both).
Consider the following two cases for $s(m_{i+1})$:
\begin{enumerate}[1.]
\item If $s(m_{i+1}) \in \{ 1,2\}$, 
then by Lemma~\ref{lem:180:li} all monomers of $\headp(m_{i})$, 
except its first monomer $m'_{i+1}$, 
lie strictly above $\ell_i$, and since $\tail(m_i)$ lies on or below $\ell_i$, 
we get that $\tail(m_i)$ does not intersect $\headp(m_{i})$, except possibly at $\pos(m'_{i+1})$. Whether $\pos(m'_{i+1})$ intersects $\tail(m_i)$ depends on the state of $m_{i-1}$:
\begin{enumerate}[(a)]
\item\label{180:a} If $s(m_{i-1}) \in\{1,2 \}$, then all monomers of $\tail(m_i)$ lie strictly below $\ell_i$ (except its first monomer $m_i$ which is not at position $\pos(m'_{i+1})$), hence $\pos(m'_{i+1})$ cannot intersect $\tail(m_i)$. 
Then $m_i$ cannot be blocked. 
\item\label{180:b} If $s(m_{i-1})=0$, then 
$m'_{i+1}$ does not intersect $\tail(m_i)$:  
Indeed, $\pos(m_{i-1}) = \pos(m_i) +\vec{x} = \pos(m'_{i+1}) + 2 \vec x \neq  \pos(m'_{i+1})$.
Furthermore, let $m_{j}$, $m_{j+1}$, ..., $m_{i-1}$ be the longest consecutive subsequence of  monomers in state 0 preceding monomer $m_i$.
Then $\pos(m_{j})$, $\pos(m_{j+1})$, ..., $\pos(m_{i+1})$ are all strictly to the west of $\pos(m_{i})$. If $j-1\ge 0$, the non-zero-state\footnote{Which must be in state 1 or 2, since 3 would give a self-intersection along the configuration.} monomer $m_{j-1}$ enforces that the monomers $m_{0}$, $m_{1}$, ..., $m_{j-1}$ lie strictly below $\ell_i$.
Thus $m_i$ is not blocked.
\end{enumerate}
Therefore, monomer $m_{i-1}$ can only be in state $3$.

\item If $s(m_{i+1}) =0$:
Both $\headp(m_i)$ and $\tail(m_i)$ have monomers on $\ell_i$, but we claim the positions of $\headp(m_i)$ do not  intersect those of $\tail(m_i)$. 
If $s(m_{i-1})\in\{1,2\}$, then all monomers of $\tail(m_i)$ except $m_i$ lie strictly below $\ell_i$, and thus $\headp(m_i)$ does not intersect $\tail(m_i)$ (and recall that $\headp(m_i)$ does not intersect $\pos(m_i)$ because configurations are simple). 
If $s(m_{i-1}) = 0$ then the monomers $M =\{m_{i-1}, m_i, m'_{i+1}\}$ lie along $\ell_i$ (pointing west). Note that a prefix of $M$ is a suffix of $\tail(m_i)$ and a (disjoint) suffix of $M$ is a prefix of $\headp(m_i)$. Hence, in order for $\tail(m_i)$ to intersect $\headp(m_i)$, one or both must depart from $\ell_i$, but, by Lemma~\ref{lem:180:li}, $\tail(m_i)$ can only do so by having monomers strictly below $\ell_i$, and $\headp(m_i)$ can only do so by having monomers strictly above $\ell_i$. Thus, monomer $m_{i-1}$ can only be in state 3. 
\end{enumerate}
Therefore, if $m_i$ is blocked, then it is in state 1 and either $m_{i-1}$ or $m_{i+1}$ is in state 3, and thus is temporarily blocked which yields Conclusion~\ref{concl:only state 1} of the lemma.
Hence, there cannot be three monomers in a row which are blocked, resulting in Conclusion~\ref{concl: no more than 2n/3 blocked} of the lemma.

For Conclusion~\ref{concl:pi:exactly 2n/3 blocked}, consider a line-rotating Turning Machine $\tml{n}{3}$ with $n=3k$ for some $k$.
The configuration $c$ with state sequence $S = (131)^{k-1} 130$ has exactly $2n/3$ blocked monomers, as every monomer in state $1$ is either blocked by a preceding monomer in state $3$, or by a following monomer in state $3$.
\end{proof}

\begin{theorem}[rotate a line by $\pi$]\label{thm:pi rotation}
For each $n\in\Nset$, the line-rotating Turning Machine $\tml{n}{3}$ computes its target configuration, and does so in expected time $O(\log n)$.
\end{theorem}
\begin{proof}
By Lemma~\ref{lem:pi:blocking}, no configuration has a permanently blocked monomer,  hence every trajectory of  $\tml{n}{3}$ ends in the target configuration. 
 
At the initial step, the rate of rule applications is $n-1$ (there are $n-1$  monomers in state~3). Over time, for successive configurations along a trajectory, 
the rate of rule applications may decrease for two reasons: 
(a) some monomers may be temporary blocked, and 
(b) after a monomer transitions to state 0 no more rules are applicable to it. 
We  reason about both: 

(a) Lemma~\ref{lem:pi:blocking}(ii) shows that a configuration with state sequence $s = (131)^{n/3-1}130$ has $2n/3$ blocked monomers, and Lemma~\ref{lem:pi:blocking}\ref{concl: no more than 2n/3 blocked} states that no configuration has more than $2n/3$ blocked monomers for $n$ divisible by 3. 
Using that fact, and in order to simplify the proof, we shall analyse a new, possibly slower, system where for any configuration $c$ that has $n' \leq n$ monomers in state $\neq 0$, we ``artificially'' block $2 n' / 3$ monomers.\footnote{The monomers are not necessarily geometrically blocked, we are merely stopping any rule from being applied to them. No configuration in a trajectory of $\tml{n}{3}$ witnesses a larger slowdown due to blocking than the slowdown we have imposed on the configurations of $T'_n$.}  
Since this assumption merely serves to slow the system, it is sufficient to give an upper bound on the expected time to finish.

(b)  A second ``slowdown'' assumption will be applied during the analysis and is justified as follows. 
Intuitively,  the number of monomers  transitioning to state 0 increases with time, and since monomers in state 0 have no applicable rules, this causes a {\em decrease} in the rate of rule applications. 
Consider a hypothetical continuous-time Markov system $M$, with $3n$ steps with rate decreasing by 1 every third step, that is, with successive rates $n,n,n,n-1,n-1,n-1,n-2,\ldots,2,1,1,1$.
By linearity of expectation, the expected value of the finishing time $T$ is the sum of the expected times $E[t_i]$ for each of the individual steps $i \in \{1,2,\ldots, 3n\}$:
\begin{equation}\label{eq:expected time 0 state pi}
E[T] = \sum_{i = 1}^{3n}E[t_i] = \sum_{m = 1}^{n}3\cdot\frac{1}{m}  = 3\sum_{m = 1}^{n}\frac{1}{m} = 3 H_n  \leq 3(\ln(n) + 1) = O(\log n)\,,
\end{equation}
where $H_n$ is the $n^{th}$ partial sum of the harmonic series $\sum_{m = 1}^{\infty}\frac{1}{m}$ with $H_n \leq \ln(n) + 1$ (see~\cite{graham1989concrete}).
Since, in $\tml{n}{3}$, it requires at least 3 steps to send a monomer from state 3 (the initial state) to state 0, no trajectory sends monomers to state 0 at a faster rate than a (hypothetical) trajectory  where a transition to state 0 appears at every third configuration (step). 
Hence, if there were no blocking whatsoever, then the expected time for $\tml{n}{3}$ would be no larger than $3 H_n$ (given by Equation~\eqref{eq:expected time 0 state pi}). 

Taking the blocking ``slowdown assumption'' in (a) into account, if the rate at step $i$ is $r_i$, then the slowed down rate is $\frac{1}{3} r_i$ giving an expected time of
\begin{equation}
E[T] = \sum_{i = 1}^{3n} E[t_i] = \sum_{m = 1}^{n}3\cdot  \frac{3}{1} \cdot \frac{1}{m}  = 9 \sum_{m = 1}^{n}\frac{1}{m} = 9 H_n  \leq 9 (\ln(n) + 1) = O(\log n)\,.
\end{equation}
Since our two assumptions merely serve to define a new system that is necessarily slower than $\tml{n}{3}$, we get the claimed expected time upper bound of $O(\log n)$ for $\tml{n}{3}$. 
\end{proof}

\subsection{Line rotation by \texorpdfstring{$4\pi/3$}{4π/3}: \texorpdfstring{$\tml{n}{4}$}{L4n}}

\begin{lemma}\label{lem:unblocked-fraction}
Let $m_i$ be a blocked monomer in some reachable configuration $c$ of a line rotation Turning Machine $\tml{n}{s}$ with $n\in\Nset$ and $1 \le s\le 4$,
and let $m_j\in\head(m_i)$ and $m_k\in\tail(m_i)$ be a pair of monomers which block the movement of $m_i$, then in the subchain of $\tml{n}{s}$ from $m_k$ to $m_{j-1}$ the number of unblocked monomers is at least half the number of blocked monomers.
\end{lemma}
\begin{proof}
Similarly to the proof of Lemma~\ref{lem:unblocked-states}, consider the closed chain $P=\pos(m_k)$, ..., $\pos(m_j)$, $\pos(m_k)$.
Let $x(m_i)$ denote the $x$-coordinate of the position of monomer $m_i$, and $y(m_i)$ denote the $y$-coordinate of the position of monomer $m_i$. 
Note, that for any $\ell$,
\begin{itemize}
\item if $s(m_\ell)=s$, then $x(m_{\ell+1})=x(m_{\ell})+1$ and $y(m_{\ell+1})=y(m_{\ell})$,
\item if $s(m_\ell)=s-1$, then $x(m_{\ell+1})=x(m_{\ell})$ and $y(m_{\ell+1})=y(m_{\ell})+1$,
\item if $s(m_\ell)=s-2$, then $x(m_{\ell+1})=x(m_{\ell})-1$ and $y(m_{\ell+1})=y(m_{\ell})+1$,
\item if $s(m_\ell)=s-3$, then $x(m_{\ell+1})=x(m_{\ell})-1$ and $y(m_{\ell+1})=y(m_{\ell})$,
\item if $s(m_\ell)=s-4$, then $x(m_{\ell+1})=x(m_{\ell})$ and $y(m_{\ell+1})=y(m_{\ell})-1$.
\end{itemize}
Let $x(m_k)-x(m_j)=\varepsilon_x$ and $y(m_k)-y(m_j)=\varepsilon_y$, with $\varepsilon_x,\varepsilon_y \in \{-1, 0, 1\}$.
The total change in $x$-coordinate and the total change in $y$-coordinate, when traversing $P$, is zero, that is,
\begin{equation}\label{eq:sum-coord}
\begin{aligned}
\sum_{\ell=k}^{j-1} (x(\ell+1)-x(\ell)) + \varepsilon_x = 0\,,\\
\sum_{\ell=k}^{j-1} (y(\ell+1)-y(\ell)) + \varepsilon_y = 0\,.
\end{aligned}
\end{equation}
Considering the first part of Equation~(\ref{eq:sum-coord}), and taking into account that the $x$-coordinate increases only when traversing monomers in state $s$, and the $x$-coordinate decreases only when traversing monomers in state $s-2$ or $s-3$, we get
$
\#(s) + \varepsilon_x = \#(s-2) + \#(s-3)
$, where $\#(u)$ denotes the number of monomers with state $u$ in the subchain from $m_k$ to $m_{j-1}$.
Observe, 
by Lemma~\ref{lem:unblocked-states}, monomers in states $s$ and $s-1$ cannot be blocked, and
since $s\le 4$, only the monomers in states $s-2$ or $s-3$ can be blocked. 
This implies, that within the subchain from $m_k$ to $m_{j-1}$, the number of blocked monomers is at most within an additive factor $1$ from the number of unblocked monomers.

Suppose, for a given subchain from $m_k$ to $m_{j-1}$, the number of monomers in state $s$ is strictly positive (that is, $\#(s)\ge 1$).
Then,
$
\#(s) \ge \frac{1}{2}(\#(s-2) + \#(s-3))
$,
that is, in the subchain, the number of unblocked monomers is at least half the number of blocked monomers.

Now suppose that the number of monomers in state $s$ in the subchain is zero (that is, $\#(s)= 0$).
As the blocked monomer $m_i$ has state either $s-2$ or $s-3$, the $x$-coordinate decreases by $1$ when traversing it.
The $x$-coordinate only increases when traversing  monomers in state $s$.
Therefore, if there are no monomers in state $s$, $\varepsilon_x$ has to be $1$, and, besides the blocked monomer $m_i$, the subchain from $m_k$ to $m_{j-1}$ consists only of monomers in states $s-4$ and $s-1$.

Furthermore, as $\varepsilon_x=1$, we have that $pos(m_k)=pos(m_j)-\vec{w}$ (that is, $m_i$ is in state $s-3$).
We claim that there is at least one monomer in state $s-1$ in the subchain from $m_k$ to $m_{j-1}$.
Indeed, consider the second part of Equation|\ref{eq:sum-coord}.
Traversing the edge between monomers $m_k$ and $m_{j}$ changes the $y$-coordinate by $\varepsilon_y=y(m_j)-y(m_k)=y(-\vec{w})=-1$.
Thus there has to be at least one monomer traversing which increases the $y$-coordinate.
This can only be a monomer in state $s-1$.
Thus, in the subchain from $m_k$ to $m_{j-1}$, there is one blocked monomer $m_i$ and at least one unblocked monomer in state $s-1$, and the total number of unblocked monomers is at least the number of blocked monomers.
\end{proof}

\begin{theorem}\label{thm:line-rotation-time}
For each $n\in\Nset$ and $1 \le s\le 4$, the line rotation Turning Machine $\tml{n}{s}$ computes its target configuration in $O(\log n)$ steps.
\end{theorem}
\begin{proof}
By Theorem~\ref{thm:line-rotation-computes} the Turning Machine $\tml{n}{s}$ computes its target configuration. That it computes the target configuration in $O(\log n)$ steps follows from the claim that in any intermediate configuration $c$, the number of blocked monomers is not greater than $3n/4$.

To prove this claim, consider a reachable configuration $c$ of $\tml{s}{n}$, and consider all blocked monomers $B=\{m_i : m_i \text{ is blocked}\}$.
\begin{figure}[t]
    \centering
    \includegraphics[page=10]{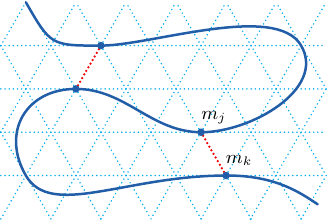}
    \caption{Subdivision $D'$ of the plane consists of chain $\tml{s}{n}$ (shown in blue), and all edges (shown in red), connecting pairs of monomers blocking the movement of some monomer, such that these edges are incident to the outer face of $D'$. }
    \label{fig:plane-subdivision}
\end{figure}
Let $e_{j,k}$ be the edge connecting the positions of two monomers $m_j$ and $m_k$ which block the movement of some monomer $m_i\in B$ (note, that $m_i$ can be blocked by more than one pair of monomers).
Let $E=\{e_{j,k}\}$ be the set of all such edges for all pairs  $m_j$ and $m_k$ which block some monomer in $\tml{s}{n}$.
Observe, that no two edges in $E$ cross each other, as they are unit segments in the triangular graph, and for the same reason no edge in $E$ crosses the chain $\tml{s}{n}$.
Let the chain $\tml{s}{n}$ together with the set of edges $E$ partition the plain into plane subdivision $D$ (refer to Figure~\ref{fig:plane-subdivision}).
The bounded faces of $D$ are formed of subchains of $\tml{s}{n}$ and edges from $E$.
Now, remove the edges of $E$ from $D$ which are not incident to the outer face, resulting in a plane subdivision $D'$.
In it, every bounded face is formed by a single subchain of $\tml{s}{n}$ and a single edge from $E$.

Observe, that all monomers of $\tml{s}{n}$ which are blocked are incident to at least one bounded face.
Otherwise, there would be two monomers $m_j$ and $m_k$ blocking the move with the edge $e_{j,k}$ not in $E$, thus contradicting the definition of $E$.

For each bounded face $f_i$ in $D'$, by Lemma~\ref{lem:unblocked-fraction}, we have 
$
\#_i(\text{unblocked}) \ge \frac{1}{2}\#_i(\text{blocked}) ,
$
where $\#_i(\text{unblocked})$ denotes the number of unblocked monomers incident to the face $f_i$, and $\#_i(\text{blocked})$ denotes the number of blocked monomers incident to the face $f_i$.

Note, that each unblocked monomer can be incident to at most two bounded faces of $D'$, and recall that each blocked monomer is incident to at least one bounded face of $D'$.
Then,
\[
\#(\text{unblocked}) \ge \frac{1}{2}\sum_{f_i\in D'} \#_i(\text{unblocked}) \ge \frac{1}{2}\left(\frac{1}{2}\sum_{f_i\in D'} \#_i(\text{blocked})\right) \ge \frac{1}{4} \#(\text{blocked})\,,
\]
where the sums are over the bounded faces of $D'$, and $\#(\text{unblocked})$ denotes the total number of unblocked monomers in $\tml{s}{n}$, and  $\#(\text{blocked})$ denotes the total number of blocked monomers in $\tml{s}{n}$.

Since there is a constant fraction of unblocked monomers in any configuration, the total expected time it takes $\tml{n}{s}$ to compute its target configuration is $O(\log n)$.
\end{proof}